\documentclass{article}
\usepackage[utf8]{inputenc}
\usepackage{authblk}
\usepackage{setspace}
\usepackage[margin=1.25in]{geometry}
\usepackage{graphicx}
\graphicspath{ {./figures/} }
\usepackage{subcaption}
\usepackage{amsmath}

\usepackage[
    style=nejm,
    citestyle=numeric-comp,
    sorting=none,
    backend=bibtex,
    natbib=true
]{biblatex}

\usepackage{hyperref}
\usepackage{url}

\usepackage[utf8]{inputenc} 
\usepackage[T1]{fontenc}   
\usepackage{amsmath,amssymb,amsfonts,amsthm,bbm}
\usepackage{mathtools}
\usepackage{microtype}
\usepackage{xcolor}
\usepackage{pifont}

\usepackage{url}
\usepackage{booktabs}
\usepackage{nicefrac}
\usepackage{enumerate,enumitem}
\usepackage{epsfig}
\usepackage{algorithm,algorithmic}
\usepackage{framed}
\usepackage{setspace}
\usepackage[title]{appendix}
\usepackage{tikz}
\usetikzlibrary{positioning,calc,arrows.meta}
\usepackage[most]{tcolorbox}
\usepackage{pifont}
\usepackage{multirow}
\usepackage{tabulary}
\usepackage{colortbl}
\usepackage{tablefootnote}
\usepackage{threeparttable}
\allowdisplaybreaks[4]

\usepackage{hyperref}

\usepackage{aliascnt}
\usepackage{makecell}
\theoremstyle{plain}
\newtheorem{theorem}{Theorem}[section]

\newaliascnt{lemma}{theorem}
\newtheorem{lemma}[lemma]{Lemma}
\aliascntresetthe{lemma}

\newaliascnt{corollary}{theorem}

\aliascntresetthe{corollary}

\newaliascnt{proposition}{theorem}

\aliascntresetthe{proposition}

\theoremstyle{definition}
\newaliascnt{definition}{theorem}
\newtheorem{definition}[definition]{Definition}
\aliascntresetthe{definition}

\newaliascnt{assumption}{theorem}
\newtheorem{assumption}[assumption]{Assumption}
\aliascntresetthe{assumption}

\theoremstyle{remark}
\newaliascnt{remark}{theorem}
\newtheorem{remark}[remark]{Remark}
\aliascntresetthe{remark}

\usepackage[capitalize,noabbrev]{cleveref}
\crefname{equation}{Eq.}{Eqs.}
\Crefname{equation}{Eq.}{Eqs.}

\crefname{theorem}{Theorem}{Theorems}
\Crefname{theorem}{Theorem}{Theorems}

\crefname{lemma}{Lemma}{Lemmas}
\Crefname{lemma}{Lemma}{Lemmas}

\crefname{corollary}{Corollary}{Corollaries}
\Crefname{corollary}{Corollary}{Corollaries}

\crefname{proposition}{Proposition}{Propositions}
\Crefname{proposition}{Proposition}{Propositions}

\crefname{definition}{Definition}{Definitions}
\Crefname{definition}{Definition}{Definitions}

\crefname{assumption}{Assumption}{Assumptions}
\Crefname{assumption}{Assumption}{Assumptions}

\crefname{remark}{Remark}{Remarks}
\Crefname{remark}{Remark}{Remarks}

\usepackage{subcaption}

\title{Sharpness-Aware Muon: Matrix-Aware Optimization with Momentum-Guided Perturbations}

\author{Zhiwei Zhai}
\author{Wenjing Yan}
\author{Ying-Jun Angela Zhang}

\affil{Department of Information Engineering,\\
The Chinese University of Hong Kong.\\
\{zz024, wjyan, yjzhang\}@ie.cuhk.edu.hk}

\date{}
\begin{document}

\maketitle

\begin{abstract}
Muon is a matrix-aware optimizer that combines momentum with orthogonalization to produce structured update directions with strong empirical behavior. However, it only shapes the geometry of the update itself and does not explicitly control neighborhood worst-case loss or bias training toward flat regions, which are often linked to stronger generalization. To address this limitation, we introduce the sharpness-aware formulation of Muon. We first develop \textsc{MuSAM}, which incorporates Sharpness-Aware Minimization (SAM)-style perturbation into Muon's orthogonalized update and establishes a principled bridge between sharpness-aware training and matrix-aware optimization. We then show that this formulation faces an efficiency--stability tension: constructing the perturbation requires an extra gradient evaluation, while the perturbation itself is generated from a single minibatch gradient and is therefore noisy. To resolve this issue, we propose \textsc{MoMuSAM}, which turns Muon's momentum from an update variable into a sharpness-probing signal. By using normalized momentum as the perturbation direction, \textsc{MoMuSAM} removes the extra clean-gradient computation and replaces the single-batch perturbation proxy with a temporally aggregated one. We prove that both \textsc{MuSAM} and \textsc{MoMuSAM} attain the same asymptotic first-order stationarity order, and further show that \textsc{MoMuSAM} yields a more stable stochastic linearization than \textsc{MuSAM}. Experiments on standard benchmarks demonstrate that sharpness awareness consistently improves Muon.
\end{abstract}

\section{Introduction}
The generalization behavior of modern deep networks depends not only on how well they fit the training data, but also on the geometry of the optimization process that selects the final solution \citep{hochreiter1994simplifying, keskar2016large}. This observation has increasingly shaped modern optimizer design along two complementary directions. One direction, exemplified by \emph{Sharpness-Aware Minimization} (SAM), improves generalization by explicitly penalizing large neighborhood worst-case loss and thereby biasing training toward flatter regions of the landscape \citep{foret2020sharpness, andriushchenko2022towards, bartlett2023dynamics}. The other reshapes the update rule itself through structured transformations of gradient information, exploiting momentum, geometry, or parameter structure to alter the inductive bias of optimization beyond plain first-order descent \citep{chen2023symbolic,yuan2024mars,vyas2024soap,pethick2025training}. Muon is a particularly compelling instance of this second direction: by coupling momentum with orthogonalization, it produces matrix-aware update directions with distinctive geometry and strong empirical optimization behavior \citep{Muon2024}.

Despite the success of both directions, however, their interaction remains poorly understood. Sharpness-aware methods primarily control the local loss geometry explored during training, whereas structured optimizers primarily control the geometry of the update used to explore it. This distinction is particularly salient in Muon, where momentum is explicitly orthogonalized into matrix-aware search directions. Such a design gives the optimizer strong control over \emph{how} it moves through parameter space, while leaving unresolved whether those structured movements are guided toward locally flat, generalization-favorable regions. In this sense, Muon does not lack geometry; rather, it leaves unmodeled the part of geometry most directly tied to flatness, namely the local worst-case loss around the iterate. Existing sharpness-aware extensions of standard momentum methods suggest that internal optimizer states can be useful for constructing perturbations \citep{becker2024momentum, li2025vasso}, but this idea has not been systematically developed for matrix-aware orthogonalized optimization. Muon therefore provides a particularly revealing setting for studying whether update geometry and local loss geometry can be coordinated, which naturally raises the following question:

\begin{tcolorbox}[colback=gray!5,
    colframe=black!70, top=3pt, bottom=3pt,
    frame style={line width=10pt}, arc=2mm]
\textbf{Can the internal state of a matrix-aware optimizer be made to govern not only how updates are taken, but also which local geometries are favored during training?}
\end{tcolorbox}

Our answer is affirmative, and it leads to a new optimizer family. We begin by formulating \textsc{MuSAM}, a sharpness-aware Muon method that couples SAM-style local worst-case perturbation with Muon's orthogonalized momentum update. This formulation plays a foundational role: it shows that explicit sharpness control is fully compatible with Muon's matrix-aware update geometry, and thereby establishes the correct starting point for sharpness-aware Muon. More importantly, it reveals the exact challenge that such an integration creates. The sharpness perturbation is expensive to construct, because it requires an additional gradient that the base Muon update does not otherwise need \citep{wen2018smoothout, zhuang2022surrogate, zhao2022stochastic}. At the same time, it is built from a \emph{single minibatch gradient}, and is therefore an inherently noisy and batch-local proxy for local sharpness. Prior work has shown that such stochastic linearization can be unstable and may fail to faithfully capture the broader local geometry that ultimately matters for generalization \citep{reddi2016stochastic, wen2022does, kim2022fisher, li2025vasso}. In this way, \textsc{MuSAM} serves a dual purpose: it provides the first principled sharpness-aware formulation of Muon, and it exposes the precise efficiency--stability tension that the remainder of the paper resolves.

The key technical insight of this paper is that resolving this tension does not require introducing an external perturbation mechanism; instead, it requires identifying the \emph{right} sharpness signal within Muon's own optimization dynamics. In particular, Muon's momentum forms a matrix-valued, temporally aggregated state that carries information across both iterations and minibatches, and is therefore typically smoother, less myopic, and more robust to batch-level noise than an instantaneous stochastic gradient \citep{li2021heavy, cutkosky2019momentum, qian2025muon}. This leads to a principled design criterion for sharpness-aware Muon: the perturbation direction should reuse the temporal information
accumulated in Muon's momentum while also serving as a stable proxy for local ascent. Guided by this criterion, we design \textsc{MoMuSAM}, in which the normalized momentum generates the perturbation and the gradient is evaluated only once at the perturbed point. Rather than merely reusing an existing optimizer state, this design connects sharpness probing with Muon's temporal momentum dynamics, turning momentum from a pure update variable into a \emph{history-aware geometric probe} of local sharpness.

This design has consequences at three levels. Algorithmically, \textsc{MoMuSAM} eliminates the extra clean-gradient evaluation required by \textsc{MuSAM}, bringing sharpness-aware Muon much closer to the computational profile of vanilla Muon. Statistically, it replaces a highly local single-batch perturbation with a temporally aggregated one, yielding a more stable slope proxy for stochastic linearization. Theoretically, this distinction is substantive rather than cosmetic: our analysis shows that \textsc{MuSAM} and \textsc{MoMuSAM} attain the same asymptotic first-order stationarity order, while \textsc{MoMuSAM} admits a sharper certified
directional-stability bound in the stochastic regime. More broadly, our results reveal a new principle for optimizer design: internal optimizer states can encode reusable geometric information, and exploiting that information can lead to genuinely stronger training algorithms.
The contributions of this paper are summarized as follows.

\begin{itemize}[topsep=0pt, itemsep=2pt, parsep=1pt, leftmargin=0.7cm]
\item We introduce, to the best of our knowledge, the first systematic sharpness-aware formulation of Muon. In particular, we establish \textsc{MuSAM}, which unifies Muon's matrix-aware orthogonalized updates with an explicit SAM-style approximation to local worst-case training and makes sharpness control an intrinsic part of the Muon framework.

\item We uncover and resolve the fundamental efficiency--stability tension in sharpness-aware Muon. Our proposed \textsc{MoMuSAM} turns Muon's momentum from an update variable into a sharpness-probing signal, thereby reducing extra gradient cost and replacing the noisy single-batch perturbation with a more stable aggregated direction.

\item We provide convergence guarantees for both \textsc{MuSAM} and
\textsc{MoMuSAM}, showing that they attain the same
\(\mathcal O(T^{-1/4})\) average first-order stationarity order for
the original objective at both the clean iterates and the corresponding
perturbed query points. Beyond matching this convergence guarantee,
we prove that \textsc{MoMuSAM} improves the certified average
directional-linearization stability from
\(\mathcal O(T^{-3/4})\) for \textsc{MuSAM} to
\(\mathcal O(T^{-1})\), while retaining only one gradient evaluation
per iteration.

\item Through experiments on standard benchmarks, we show that introducing sharpness awareness consistently strengthens Muon, while \textsc{MoMuSAM} delivers a markedly better generalization--efficiency tradeoff than \textsc{MuSAM}. Taken together, these results position sharpness-aware Muon as a new optimizer family that tightly integrates structured matrix updates, local geometry control, and computational efficiency.
\end{itemize}

\section{Preliminaries}
\label{sec:prelim}

\subsection{Problem Formulation}
\label{subsec:problem}

We consider stochastic nonconvex optimization over a matrix-valued parameter
\(
W \in \mathbb{R}^{m \times n}
\),
with objective
\begin{equation}
    \min_{W \in \mathbb{R}^{m \times n}} f(W)
    \;:=\;
    \mathbb{E}_{\xi}\big[\ell(W;\xi)\big],
    \label{eq:prelim_problem}
\end{equation}
where \(\xi\) denotes a random sample (or minibatch) drawn from the data distribution, and
\(\ell(W;\xi)\) is the corresponding stochastic loss.
At iteration \(t\), given a minibatch \(\mathcal{B}_t\), we write the stochastic gradient as $\nabla \ell(W_t;\mathcal{B}_t).$
For simplicity of presentation, the analysis below is developed for a single matrix block \(W\); in a neural network with multiple matrix-valued hidden-layer parameters, the same construction is applied blockwise.

\paragraph{Notation.}
For a matrix \(X \in \mathbb{R}^{m \times n}\), \(\|X\|_F\) denotes the Frobenius norm,
\(\|X\|_\ast\) denotes the nuclear norm, and \(\|X\|_{\mathrm{op}}\) denotes the spectral (operator) norm.
We use the standard inequality $\|X\|_\ast \le \sqrt{r}\,\|X\|_F$,
where \(r\) is an upper bound on \(\mathrm{rank}(X)\). Without loss of generality, we assume $m \geq n$, so the rank of the matrix is at most $n$.
For \(X,Y \in \mathbb{R}^{m \times n}\), we define the matrix inner product
$\langle X,Y\rangle := \mathrm{Tr}(X^\top Y)$,
where \(\mathrm{Tr}(\cdot)\) denotes the trace.
We write \(\mathrm{Polar}(X)\) for the polar factor of \(X\), and use
\(\mathrm{NS\_Orth}(X)\) to denote its practical Newton--Schulz
approximation. For notational simplicity, whenever the minibatch
\(\mathcal B_t\) is clear from context, we write
\(\nabla\ell(W_t)\) as shorthand for
\(\nabla\ell(W_t;\mathcal B_t)\). Whenever a normalized direction
\(V/\|V\|_F\) appears below, it is understood to be zero when
\(V=0\). We reserve \(\eta>0\) for the learning rate,
\(\mu\in[0,1)\) for the momentum coefficient, and \(\rho>0\) for
the sharpness perturbation radius.

In this paper, the following standard assumptions from stochastic nonconvex optimization  \citep{ghadimi2013stochastic, bottou2018optimization, mi2022make, zhuang2022surrogate, li2025vasso, kim2026convergence} are made for the convergence analysis:

\begin{assumption}[Lower bounded objective]
\label{ass:lower_bound}
The objective \(f\) is bounded from below, i.e., there exists \(f_\star > -\infty\) such that $f(W) \ge f_\star, \forall W \in \mathbb{R}^{m \times n}.$
\end{assumption}

\begin{assumption}[Sample-wise Lipschitz smoothness]
\label{lipschiz}
For almost every sample \(\xi\), the stochastic loss
\(\ell(\cdot;\xi)\) is continuously differentiable and satisfies
\begin{equation}
\left\|
\nabla\ell(X;\xi)-\nabla\ell(Y;\xi)
\right\|_*
\le
L\|X-Y\|_{\mathrm{op}},
\qquad
\forall X,Y\in\mathbb R^{m\times n}.\notag
\end{equation}
\end{assumption}
By Jensen's inequality, \cref{lipschiz} implies that the same
smoothness bound holds almost surely for every sampled minibatch
gradient and also for the population gradient.

\begin{assumption}[Conditionally unbiased stochastic gradients]
\label{gradientvariance}
Let \(\mathcal F_{t-1}\) denote the history before sampling the
minibatch at iteration \(t\). The minibatch \(\mathcal B_t\),
consisting of \(B\) independent samples, is sampled independently
of \(\mathcal F_{t-1}\). For every
\(\mathcal F_{t-1}\)-measurable query point \(X_t\), we have
\[
\mathbb E\!\left[
\nabla\ell(X_t;\mathcal B_t)
\mid
\mathcal F_{t-1}
\right]
=
\nabla f(X_t), \quad\mathbb E\!\left[
\left\|
\nabla\ell(X_t;\mathcal B_t)-\nabla f(X_t)
\right\|_F^2
\mid
\mathcal F_{t-1}
\right]
\le
\frac{\sigma^2}{B}.
\]
\end{assumption}

\cref{lipschiz} is naturally expressed with respect to a norm and its dual norm; in the matrix setting, the operator--nuclear pair is especially compatible with orthogonalized or polar-type updates \citep{nesterov2013introductory,jaggi2013revisiting, beck2017first, kim2026convergence}. \cref{gradientvariance} is equally standard and captures the stochastic noise introduced by minibatch sampling, under which the variance of a minibatch estimator decreases proportionally to the batch size \citep{ghadimi2013stochastic, bottou2018optimization}.
Under these assumptions, we use the standard notion of first-order stationarity.

\begin{definition}[\(\varepsilon\)-stationary point]
\label{def:stationary}
A point \(W \in \mathbb{R}^{m \times n}\) is called an \(\varepsilon\)-stationary point if
$\|\nabla f(W)\|_* \le \varepsilon.$
For an iterative algorithm producing \(\{W_t\}_{t=0}^{T-1}\), we say that it attains \(\varepsilon\)-stationarity in \(T\) iterations if $\frac{1}{T}\sum_{t=0}^{T-1}\mathbb{E}\!\left[\|\nabla f(W_t)\|_*\right]
\le \varepsilon$.
\end{definition}

\subsection{Momentum Orthogonalization}
\label{subsec:muon}

Recent work on orthogonalized optimization, exemplified by Muon, advocates exploiting the intrinsic matrix structure of neural network parameters instead of flattening them into long vectors. The central idea is to first construct a matrix-valued momentum term and then orthogonalize it before taking an update step. This procedure preserves the dominant directional subspaces of the momentum while removing singular-value magnitudes, thereby yielding a structured search direction adapted to matrix parameters \citep{liu2025muon, an2025asgo, liu2025cosmos, shah2025practical}. 

For exposition, we assume \(m\ge n\); when \(m<n\), the polar
operation is applied to \(M_t^\top\), and the resulting direction is
transposed back. Denote \(M_t \in \mathbb{R}^{m\times n}\) as the momentum matrix and \(O_t \in \mathbb{R}^{m\times n}\) as the orthogonalized update direction. A generic Muon-style update takes the form
\begin{equation}
    M_t
    \;=\;
    \mu M_{t-1} + (1-\mu)\nabla \ell(W_t),
    \qquad \mu \in [0,1),
    \label{eq:muon_momentum_prelim}
\end{equation}
followed by
\begin{equation}
    O_t
    \;\in\;
    \arg\min_{O \in \mathbb{R}^{m\times n}}
    \|O-M_t\|_F
    \quad
    \text{s.t.}\quad
    O^\top O = I_n,
    \label{eq:muon_orth_problem_prelim}
\end{equation}
and the parameter update
\begin{equation}
    W_{t+1}
    \;=\;
    W_t - \eta O_t.
    \label{eq:muon_update_prelim}
\end{equation}
Thus, Muon updates along the closest column-orthonormal matrix to the momentum.

If \(M_t = U_t \Sigma_t V_t^\top\) is the thin singular value decomposition (SVD), then the solution to \cref{eq:muon_orth_problem_prelim} is the polar factor $O_t=\mathrm{Polar}(M_t)=U_tV_t^\top$.
In practice, this orthogonalization is not computed by exact SVD; Muon instead uses a few Newton--Schulz iterations \citep{Muon2024} to obtain an efficient SVD-free approximation,
$O_t \approx \mathrm{NS\_Orth}(M_t)$,
which relies mainly on matrix multiplications and is therefore much cheaper at scale.

Muon is compelling because it is explicitly matrix-aware: it does not merely accumulate gradients, but transforms momentum into a spectrally normalized update direction. However, this geometric normalization is still \emph{sharpness-agnostic}. It shapes the update direction, yet leaves uncontrolled the local worst-case loss around the iterate. As a result, Muon does not explicitly encourage convergence to flat regions, and any improvement in generalization arises only implicitly. This missing link is precisely what makes sharpness-aware Muon both important and nontrivial: it requires enriching a structured matrix optimizer with an explicit mechanism for local geometry control. \citep{keskar2016large, foret2020sharpness, li2025vasso}

\section{Algorithm Design}
\label{sec:method}

\subsection{A Sharpness-Aware Extension of Muon: MuSAM}
\label{subsec:musam}

We start from the most direct way to make Muon sharpness-aware. Since Muon already provides a structured, orthogonalized update rule for matrix-valued parameters, a natural next step is to explicitly incorporate local worst-case geometry into its update. To this end, we consider a single matrix parameter block
\(
W \in \mathbb{R}^{m\times n}
\)
and construct a sharpness-aware Muon update at each iteration. The same design is applied blockwise to all matrix-valued parameters in a full network.

\paragraph{Step 1: a local minimax objective.}
We define the local sharpness-aware objective
\begin{equation}
    \min_{W}\max_{\|\delta\|_F \le \rho}\ f(W+\delta),
    \label{eq:musam_minimax}
\end{equation}
where \(\rho>0\) is the perturbation radius.
The inner maximization searches for a worst-case perturbation inside a small Frobenius ball, and therefore explicitly measures the local sensitivity of the loss around the current iterate. This is precisely the mechanism through which sharpness-aware training biases optimization toward flatter regions of the loss landscape \citep{keskar2016large, li2021heavy, foret2020sharpness, andriushchenko2022towards}.

\paragraph{Step 2: stochastic linearization of the inner maximization.}
Directly solving the inner problem in \cref{eq:musam_minimax} is computationally intractable for deep networks. Following the SAM principle \citep{li2025vasso, becker2024momentum}, we approximate it by a first-order stochastic linearization:
\begin{equation}
\begin{aligned}
    \delta_t
    &=
    \arg\max_{\|\delta\|_F \le \rho}\ f(W_{t-1}+\delta)\notag\\
    &\approx
    \arg\max_{\|\delta\|_F \le \rho}
    \Big(
        f(W_{t-1})
        +
        \langle \nabla f(W_{t-1}), \delta \rangle
    \Big)
    \approx
    \arg\max_{\|\delta\|_F \le \rho}
    \Big(
        f(W_{t-1})
        +
        \langle \nabla \ell(W_{t-1}), \delta \rangle
    \Big),
\end{aligned}
\label{eq:musam_linearization}
\end{equation}
where the first approximation is the local first-order Taylor expansion of \(f(W_{t-1}+\delta)\), which is accurate when the perturbation radius \(\rho\) is small. The second replaces the full gradient \(\nabla f(W_{t-1})\) by the minibatch stochastic gradient \(\nabla \ell(W_{t-1})\), yielding a computationally tractable adversarial direction. By Cauchy--Schwarz, the maximizer of
\cref{eq:musam_linearization} admits the closed form
\begin{equation}
\delta_t
=
\rho\,
\frac{\nabla\ell(W_{t-1};\mathcal B_t)}{\|\nabla\ell(W_{t-1};\mathcal B_t)\|_F}.
\label{eq:musam_delta}
\end{equation}
Thus, the perturbation is the normalized
stochastic ascent direction.

\paragraph{Step 3: Muon update from the perturbed point.}
Once the perturbation is obtained, we evaluate the gradient at the perturbed iterate:
$\nabla\ell(W_{t-1}+\delta_t;\mathcal B_t)$,
and use it to update the momentum matrix:
\begin{equation}
    M_t
    \;=\;
    \mu M_{t-1} + (1-\mu)\nabla\ell(W_{t-1}+\delta_t;\mathcal B_t),
    \label{eq:musam_momentum}
\end{equation}
The orthogonalization and parameter update then follow the standard Muon pipeline introduced in \cref{eq:muon_orth_problem_prelim,eq:muon_update_prelim}. 
This yields a direct sharpness-aware Muon method, which we call \textsc{MuSAM}, as summarized in \cref{alg:musam}.

\paragraph{Why this construction is the right starting point.}
The above design is not an arbitrary hybrid. It is the most faithful sharpness-aware extension of Muon. The perturbation \(\delta_t\) is obtained from a stochastic
first-order approximation of the local inner maximization, and thus
serves as an explicit SAM-style probe of local worst-case geometry. The outer update, in contrast, remains purely Muon-style: the perturbed gradient is accumulated into a momentum matrix and then orthogonalized before the parameter update. In this way, \textsc{MuSAM} cleanly separates two roles: the perturbation probes local geometry, while the orthogonalized Muon step preserves the matrix-aware update structure. This makes \textsc{MuSAM} a principled baseline for asking whether explicit sharpness control can improve Muon.

\begin{algorithm}[t]
\caption{Sharpness-Aware Muon (\textsc{MuSAM})}
\label{alg:musam}
\begin{algorithmic}[1]
\REQUIRE Loss \(\ell(W)\), initial matrix parameter \(W_0 \in \mathbb{R}^{m\times n}\), learning rate \(\eta\), momentum coefficient \(\mu\in[0,1)\), perturbation radius \(\rho>0\)
\STATE Initialize \(M_0 \gets 0\)
\FOR{\(t = 1,2,\dots\)}
    \STATE Compute clean stochastic gradient $\nabla\ell(W_{t-1};\mathcal B_t)$
    \STATE Construct sharpness-aware perturbation $\delta_t \gets \rho\, \frac{\nabla\ell(W_{t-1};\mathcal B_t)}{\|\nabla\ell(W_{t-1};\mathcal B_t)\|_F}$
    \STATE Compute perturbed stochastic gradient $\nabla\ell(W_{t-1}+\delta_t;\mathcal B_t)$
    \STATE Update momentum matrix $M_t \gets \mu M_{t-1} + (1-\mu)\nabla \ell(W_{t-1}+\delta_t;\mathcal B_t)$
    \STATE Orthogonalize the momentum matrix $O_t \in \arg\min_{O\in\mathbb{R}^{m\times n}} \|O-M_t\|_F,
    \quad\!\!\!\!\text{s.t.}\quad\!\!\!\!
    O^\top O = I_n$
    \STATE Update parameter $W_t \gets W_{t-1} - \eta O_t$
\ENDFOR
\end{algorithmic}
\end{algorithm}

\paragraph{Practical considerations.}
\textsc{MuSAM} already establishes an important point: sharpness awareness can be introduced into Muon in a principled way, without giving up Muon's matrix-aware inductive bias. At the same time, this formulation exposes two practical issues. First, the perturbation in \cref{eq:musam_delta} requires a clean gradient at \(W_{t-1}\), while the actual update uses a second gradient at \(W_{t-1}+\delta_t\); this substantially increases the per-step cost, just as in vanilla SAM \citep{wen2018smoothout, zhuang2022surrogate, zhao2022stochastic}. Second, the perturbation direction is determined by a single minibatch gradient. In stochastic training, such a direction can be overly batch-local and sensitive to sampling noise, so the resulting adversarial probe may not faithfully reflect the broader local geometry that matters for generalization. Recent analyses describe this phenomenon as instability of stochastic linearization, or equivalently, as the emergence of a {friendly adversary} that weakens the intended sharpness signal \citep{becker2024momentum, li2025vasso}. 

\subsection{A Momentum-Guided Refinement: \textsc{MoMuSAM}}
\label{subsec:momusam}

We now turn Muon's momentum from a pure update variable into a sharpness-probing signal.

\paragraph{Core idea.}
The key observation is that Muon already maintains a matrix-valued momentum state \(M_{t-1}\) before the current backward pass. This state is available at no extra cost and aggregates gradient
information across previous iterations. We therefore use \(M_{t-1}\) as the slope proxy for constructing the local sharpness perturbation, rather than recomputing a fresh clean gradient at \(W_{t-1}\).

\paragraph{Momentum-guided local linearization.}
We use \(M_{t-1}\) as a \emph{proxy} for the local ascent direction when constructing the perturbation. This yields the momentum-guided perturbation
\begin{equation}
    \delta_t=\rho\dfrac{M_{t-1}}{\|M_{t-1}\|_F}.
    \label{eq:momusam_delta}
\end{equation}

Here the perturbation reuses Muon's matrix-valued momentum, while the
matrix-aware geometry is retained by the polar outer update.

\paragraph{Why momentum yields a more stable perturbation proxy.}
The key reason is that the momentum state aggregates gradient
information over time, rather than relying on a single minibatch
gradient. Starting from the recursion
\(
M_t
=
\mu M_{t-1}
+
(1-\mu)
\nabla\ell(W_{t-1}+\delta_t;\mathcal B_t),
\)
one can unroll it from \(M_0=0\) to obtain
\begin{equation}
M_{t-1}
=
(1-\mu)
\sum_{s=1}^{t-1}
\mu^{t-1-s}
\nabla\ell(W_{s-1}+\delta_s;\mathcal B_s).
\label{eq:momusam_ema}
\end{equation}

Thus, \(M_{t-1}\) combines perturbed stochastic gradients from
multiple iterations with exponentially decaying weights, reducing
its sensitivity to the fluctuations of any single minibatch.
Because these gradients are evaluated along a changing optimization
trajectory, the resulting stability depends on both stochastic-noise
averaging and momentum-tracking error. By controlling these two
effects jointly, \cref{thm:linearization_stability} shows that, under
the prescribed horizon-dependent hyperparameters, the certified
average achieved-linearization gap improves from
\(\mathcal O(T^{-3/4})\) for the single-batch perturbation of
\textsc{MuSAM} to \(\mathcal O(T^{-1})\) for the momentum-guided
perturbation of \textsc{MoMuSAM}. Hence, \(M_{t-1}\) provides a
smoother and more history-aware sharpness probe without being
interpreted as an exact solution of the inner maximization
\citep{becker2024momentum,li2025vasso}.

\paragraph{Update rule.}
After constructing the perturbation in \cref{eq:momusam_delta}, we evaluate the gradient at the perturbed point and then proceed with the standard Muon update, following \cref{eq:muon_orth_problem_prelim,eq:muon_update_prelim,eq:musam_momentum}. The resulting method is termed \textsc{MoMuSAM}, short for \emph{Momentum-guided Sharpness-Aware Muon}; its full procedure is summarized in \cref{alg:momusam} in \cref{subsectionb3}.

\begin{table}[t]
\centering
\caption{Comparison between
\textsc{Muon}, \textsc{MuSAM}, and \textsc{MoMuSAM}.}
\label{tab:momusam_advantages}
\vspace{2mm}
\begin{tabular}{lccc}
\specialrule{1.9pt}{2pt}{2pt}
Aspect
& \textsc{Muon}
& \textsc{MuSAM}
& \textsc{MoMuSAM} \\
\midrule
Perturbation source
& --
& $\nabla \ell(W_{t-1})$
& $M_{t-1}$ \\

Per-iteration gradient cost
& 1
& 2
& 1 \\

Directional stability$^\dagger$
& --
& $\mathcal O(T^{-3/4})$
& $\mathcal O(T^{-1})$ \\

First-order stationarity$^\diamond$
& $\mathcal O(T^{-1/4})$
& $\mathcal O(T^{-1/4})$
& $\mathcal O(T^{-1/4})$ \\
\specialrule{1.9pt}{2pt}{2pt}
\end{tabular}

\vspace{2.5mm}
\begin{minipage}{0.95\linewidth}
\footnotesize
$^\dagger$ The average directional-stability bounds for
\textsc{MuSAM} and \textsc{MoMuSAM} are established in
\cref{thm:linearization_stability}.\\
$^\diamond$ The stationarity rate of \textsc{Muon} follows from
\citet{shen2025convergence}; those of \textsc{MuSAM} and
\textsc{MoMuSAM} are established in
\cref{thm:musam_main,thm:momusam_main}, respectively.
\end{minipage}
\end{table}

\paragraph{What \textsc{MoMuSAM} buys us.}
The key differences among \textsc{Muon}, \textsc{MuSAM}, and
\textsc{MoMuSAM} are summarized in
\cref{tab:momusam_advantages}. Relative to the direct
sharpness-aware formulation, \textsc{MoMuSAM} provides three
complementary advantages.
\begin{itemize}[
    topsep=0pt,
    itemsep=3pt,
    parsep=1pt,
    leftmargin=0.65cm
]
\item
First, \textsc{MoMuSAM} eliminates the additional clean-gradient
evaluation required by \textsc{MuSAM}. It therefore reduces the
per-iteration gradient cost from two evaluations to one, recovering
the computational profile of vanilla Muon.

\item
Second, \textsc{MoMuSAM} replaces the instantaneous single-batch
perturbation of \textsc{MuSAM} with the temporally aggregated
momentum state. As established in
\cref{thm:linearization_stability}, this improves the certified
average achieved-linearization gap from
\(\mathcal O(T^{-3/4})\) to \(\mathcal O(T^{-1})\) under the same
horizon-dependent hyperparameters. Importantly, this sharper result
controls the complete direction-sensitive stability error, including
both stochastic fluctuation and momentum-tracking error, rather than
only an isolated variance component.

\item
Third, \textsc{MoMuSAM} obtains this stability improvement without
introducing an auxiliary estimator or optimizer state. At iteration
\(t\), the momentum \(M_t\) is orthogonalized to produce the current
Muon update and is subsequently normalized to construct the next
sharpness perturbation \(\delta_{t+1}\). The information accumulated
for optimization is therefore immediately reused for neighborhood
probing, tightly coupling sharpness awareness with Muon's temporal
state evolution without adding another backward pass.
\end{itemize}

\paragraph{Design philosophy.}
The key idea of \textsc{MoMuSAM} is to construct the perturbation from a signal that is less sensitive to the randomness of a single minibatch. Prior studies on momentum-guided and stabilized SAM variants suggest that such smoother adversarial directions can better capture the local geometry relevant to generalization under stochastic training \citep{becker2024momentum, li2025vasso}. What has remained missing, however, is an explicit theoretical characterization of this advantage. The next section fills this gap by showing that the momentum-guided perturbation in \textsc{MoMuSAM} yields a more stable stochastic linearization while preserving the same convergence order as \textsc{MuSAM}.

\section{Theoretical Analysis}
\label{sec:theory}

We now show that both \textsc{MuSAM} and \textsc{MoMuSAM} retain favorable convergence guarantees under the stochastic matrix-valued setting in \cref{sec:prelim}. Our analysis addresses two complementary questions. First, do the proposed sharpness-aware Muon variants converge to first-order stationary points, both at the clean iterate and at the perturbed iterate? Second, does the momentum-guided perturbation in \textsc{MoMuSAM} enjoy a genuine stability advantage over the raw-gradient perturbation used in \textsc{MuSAM}? The first question is answered by \cref{thm:musam_main,thm:momusam_main}, while the second is addressed by \cref{thm:linearization_stability}.
Throughout this section, let $r := \min\{m,n\}$. All proofs are provided in the appendix. 

\begin{theorem}[Convergence of \textsc{MuSAM}]
\label{thm:musam_main}
Suppose \cref{ass:lower_bound,lipschiz,gradientvariance} hold, and let \(\{W_t\}_{t=0}^{T-1}\) be generated by \textsc{MuSAM}. Then
\begin{align}
\frac{1}{T}\sum_{t=0}^{T-1} \mathbb E\|\nabla f(W_{t})\|_*
&\!\lesssim\!
\frac{f(W_0)\!-\!f_\star}{\eta T}
\!+\!
\frac{\sqrt{r}}{(1\!-\!\mu^2)T}
\sqrt{
\mu^2L (f(W_0)\!-\!f_\star)
\!+\!
(1\!-\!\mu)^2L^2\rho^2
\!+\!
(1\!-\!\mu)^2\frac{\sigma^2}{B}
}
\nonumber\\
&\quad\!+\!
\frac{\sqrt{r(1\!+\!\mu^2)}}{1\!-\!\mu^2}\!L\!\bigl(\mu\eta\!+\!(1\!-\!\mu)\rho\bigr)
\!+\!\!
\sqrt{\frac{r(1\!-\!\mu)}{1+\mu}}\!\left(\!\!
\frac{\sigma}{\sqrt B}\!+\!L\rho
\!\!\right)\!+\!
(1\!+\!\sqrt{r})L\rho
\!+\!
{L\eta}.
\end{align}
Moreover, if the hyperparameters are chosen as $\eta=\rho=T^{-3/4}$ and
$1-\mu=T^{-1/2}$,
then $\frac{1}{T}\sum_{t=0}^{T-1}\mathbb E\|\nabla f(W_t)\|_*
= \mathcal{O}(1/{T^\frac{1}{4}})$ and $\frac{1}{T}\sum_{t=0}^{T-1}\mathbb E\|\nabla f(W_t+\delta_{t+1})\|_*
=\mathcal{O}(1/{T^\frac{1}{4}}).$
\end{theorem}

\begin{theorem}[Convergence of \textsc{MoMuSAM}]
\label{thm:momusam_main}
Suppose \cref{ass:lower_bound,lipschiz,gradientvariance} hold, and let \(\{W_t\}_{t=0}^{T-1}\) be generated by \textsc{MoMuSAM}. Then
\begin{align}
\frac{1}{T}\sum_{t=0}^{T-1} \mathbb E\|\nabla f(W_{t})\|_*
&\!\lesssim\!
\frac{f(W_0)\!-\!f_*}{\eta T}
\!+\!\frac{\sqrt r}{(1\!-\!\mu^2)T}\sqrt{
\mu^2L(f(W_0)\!-\!f_*)\!+\!
(1\!-\!\mu)^2\frac{\sigma^2}{B}}\notag\\
&\quad+
\frac{\sqrt{r(1\!+\!\mu^2)}}{1\!-\!\mu^2}\,L(\eta\!+\!\rho)\!+\!
\sqrt{\frac{r(1\!-\!\mu)}{1+\mu}}\;\frac{\sigma}{\sqrt B}
\!+\!L\rho
\!+\!{L\eta}.
\end{align}
Moreover, if the hyperparameters are chosen as $\eta=\rho=T^{-3/4}$ and
$1-\mu=T^{-1/2}$,
then $\frac{1}{T}\sum_{t=0}^{T-1}\mathbb E\|\nabla f(W_t)\|_*
=\mathcal{O}(1/{T^\frac{1}{4}})$ and $\frac{1}{T}\sum_{t=0}^{T-1}\mathbb E\|\nabla f(W_t+\delta_{t+1})\|_*
= \mathcal{O}(1/{T^\frac{1}{4}}).$
\end{theorem}

\begin{remark}[Same asymptotic rate, different algorithmic role]
\label{rem:same_rate}
\cref{thm:musam_main,thm:momusam_main} show that the momentum-guided refinement in \textsc{MoMuSAM} does \emph{not} pay an optimization penalty in terms of asymptotic first-order stationarity: both \textsc{MuSAM} and \textsc{MoMuSAM} attain the same order at the clean iterate \(W_t\) and at the perturbed iterate \(W_t+\delta_{t+1}\). Thus, replacing the raw-gradient perturbation by a momentum-guided one is not a heuristic shortcut that sacrifices convergence guarantees.
\end{remark}

\begin{remark}[Where \textsc{MoMuSAM} improves]
\label{rem:momusam_advantage}
While \textsc{MuSAM} and \textsc{MoMuSAM} are of the same asymptotic order, \textsc{MoMuSAM} enjoys two additional advantages. First, it removes the extra clean-gradient evaluation required to construct the perturbation in \textsc{MuSAM}, thereby reducing the per-iteration gradient cost. Second, as shown in \cref{thm:linearization_stability}, it uses a more stable, temporally aggregated slope proxy for stochastic linearization. Hence, \textsc{MoMuSAM} preserves the same stationarity guarantee as \textsc{MuSAM} while offering a more efficient and statistically more stable perturbation mechanism.
\end{remark}

\paragraph{Stability of stochastic linearization.}
To evaluate the perturbation actually induced by a slope proxy, we use a
direction-sensitive population linearization criterion. For every nonzero
\(V\in\mathbb R^{m\times n}\), let
\(\mathsf n(V):=V/\|V\|_F\), and set \(\mathsf n(0):=0\).

\begin{definition}[Directional linearization quality]
\label{def:delta_stability}
Let \(g_t:=\nabla f(W_{t-1})\). Define
\begin{equation}
L_t(V)
:=
f(W_{t-1})
+
\rho\langle g_t,\mathsf n(V)\rangle,
\qquad
\mathcal E_t(V)
:=
L_t(g_t)-L_t(V).
\label{eq:directional_linearization_gap}
\end{equation}
A slope proxy \(V_t\) is said to be \(\Delta_t\)-stable if
\(
\mathbb E[\mathcal E_t(V_t)]\le\Delta_t.
\)
\end{definition}

The gap \(\mathcal E_t(V)\) measures the loss in population
linearization caused by using the perturbation direction induced by \(V\)
instead of the population-gradient direction. In particular,
\(\mathcal E_t(V)\ge0\) and
\begin{equation}
\mathcal E_t(V)
\le
2\rho\|V-\nabla f(W_{t-1})\|_F.
\label{eq:directional_gap_basic}
\end{equation}

\begin{theorem}[Stability of stochastic linearization]
\label{thm:linearization_stability}
Suppose
\cref{ass:lower_bound,lipschiz,gradientvariance}
hold.

\paragraph{(i) \textsc{MuSAM}.}
For the raw-gradient perturbation used by \textsc{MuSAM},
\begin{equation}
\mathbb E\!\left[
\mathcal E_t\!\left(
\nabla\ell(W_{t-1};\mathcal B_t)
\right)
\right]
\le
\frac{2\rho\sigma}{\sqrt B}.
\label{eq:musam_stability}
\end{equation}

\paragraph{(ii) \textsc{MoMuSAM}.}
For every \(t\ge2\), the momentum-guided perturbation used by
\textsc{MoMuSAM} satisfies
\begin{equation}
\mathbb E\!\left[
\mathcal E_t(M_{t-1})
\right]
\le
2\rho
\left[
L\left(
\rho+\frac{\eta}{1-\mu}
\right)
+
\sqrt{
\frac{1-\mu}{1+\mu}
\frac{1+\mu^{t-1}}{1-\mu^{t-1}}
}
\frac{\sigma}{\sqrt B}
\right].
\label{eq:momusam_stability}
\end{equation}

Moreover, if
\(
\eta=\rho=T^{-3/4}
\)
and
\(
1-\mu=T^{-1/2},
\)
then
\begin{align}
\frac{1}{T}
\sum_{t=1}^{T}
\mathbb E\!\left[
\mathcal E_t\!\left(
\nabla\ell(W_{t-1};\mathcal B_t)
\right)
\right]
=
\mathcal O\!\left(
\frac{\sigma}{\sqrt B\,T^{3/4}}
\right),
\frac{1}{T}
\sum_{t=1}^{T}
\mathbb E\!\left[
\mathcal E_t(M_{t-1})
\right]
=
\mathcal O\!\left(
\frac{L+\sigma/\sqrt B}{T}
\right).
\label{eq:momusam_stability_rate}
\end{align}
\end{theorem}

\begin{remark}[Sharper certified stochastic stability]
\label{rem:momusam_stability_advantage}
In the stochastic regime with minibatch noise,
\cref{thm:linearization_stability} establishes a sharper certified
asymptotic upper-bound order for the stochastic-linearization error
of \textsc{MoMuSAM}. Under the same horizon-dependent
hyperparameters, the single-batch perturbation of \textsc{MuSAM}
admits an average stability bound of order
\(\mathcal O(T^{-3/4})\), whereas the temporally aggregated
perturbation of \textsc{MoMuSAM} improves this order to
\(\mathcal O(T^{-1})\). The improvement follows from averaging
stochastic gradient noise across iterations, while the accompanying
momentum-tracking error remains controlled at the same
\(\mathcal O(T^{-1})\) scale. Thus, the advantage is obtained for the
complete stability bound, rather than only for an isolated variance
term.
\end{remark}

\begin{remark}[Sharper stability without an optimization penalty]
\label{rem:momusam_theoretical_picture}
Taken together,
\cref{thm:momusam_main,thm:linearization_stability}
show that \textsc{MoMuSAM} preserves the
\(\mathcal O(T^{-1/4})\) first-order stationarity order of
\textsc{MuSAM}, while improving the certified average
stochastic-linearization stability from
\(\mathcal O(T^{-3/4})\) to \(\mathcal O(T^{-1})\).
Moreover, \textsc{MoMuSAM} achieves this sharper stability guarantee
using only one gradient evaluation per iteration, thereby retaining
the per-iteration gradient cost of vanilla Muon.
\end{remark}

\section{Experiments}
\label{sec:experiments}

We evaluate the proposed \textsc{MuSAM} and \textsc{MoMuSAM} against a diverse set of baselines spanning both sharpness-aware and matrix-aware optimization. On the sharpness-aware side, we compare with standard SAM~\citep{foret2020sharpness}, MSAM~\citep{becker2024momentum}, and VASSO~\citep{li2025vasso}, which represent the canonical SAM formulation, a momentum-guided sharpness-aware variant, and a variance-suppressed stabilization approach, respectively. On the matrix-aware side, we compare with Muon~\citep{Muon2024}, the most relevant baseline for assessing whether explicit sharpness control can enhance matrix-aware orthogonalized optimization.
Our experiments cover both image classification and language modeling tasks. For image classification, we evaluate on Fashion-MNIST (FMNIST)~\citep{xiao2017fashion} using a multilayer perceptron (MLP), and on CIFAR-10 and CIFAR-100~\citep{krizhevsky2009learning} using CifarNet and ResNet-18, respectively. For language modeling, we adopt the NanoGPT framework~\citep{karpathy2022nanoGPT} to reproduce a 124M-parameter GPT-2 model~\citep{radford2019language} trained on OpenWebText~\citep{Gokaslan2019OpenWeb}.

In \cref{fig:accuracy4algorithm}, we compare the convergence performance of different optimizers across FMNIST, CIFAR-10, CIFAR-100, and OpenWebText. The results show that introducing sharpness awareness into Muon consistently improves the final performance, with \textsc{MuSAM} already outperforming the baselines in most cases and \textsc{MoMuSAM} delivering the strongest overall results. This trend can be understood from the complementary roles of the two components: Muon improves the \emph{update geometry} through matrix-aware momentum orthogonalization, whereas sharpness-aware perturbation improves the \emph{local loss geometry} explored during training by biasing optimization toward regions with smaller neighborhood worst-case loss. The further gain of \textsc{MoMuSAM} over \textsc{MuSAM} shows that the quality of the perturbation signal is also crucial. \textsc{MuSAM} builds its perturbation from a single minibatch gradient, so it remains sensitive to stochastic noise and batch-local fluctuations. \textsc{MoMuSAM}, in contrast, uses the momentum state, which aggregates information across iterations and minibatches. This makes the perturbation direction smoother, more stable, and more informed by Muon's accumulated optimization history, leading to a consistently better tradeoff between convergence quality, generalization, and computational efficiency.

\begin{figure*}[t]
\centering
\begin{minipage}[b]{0.2448\textwidth}
    \includegraphics[width=\textwidth]{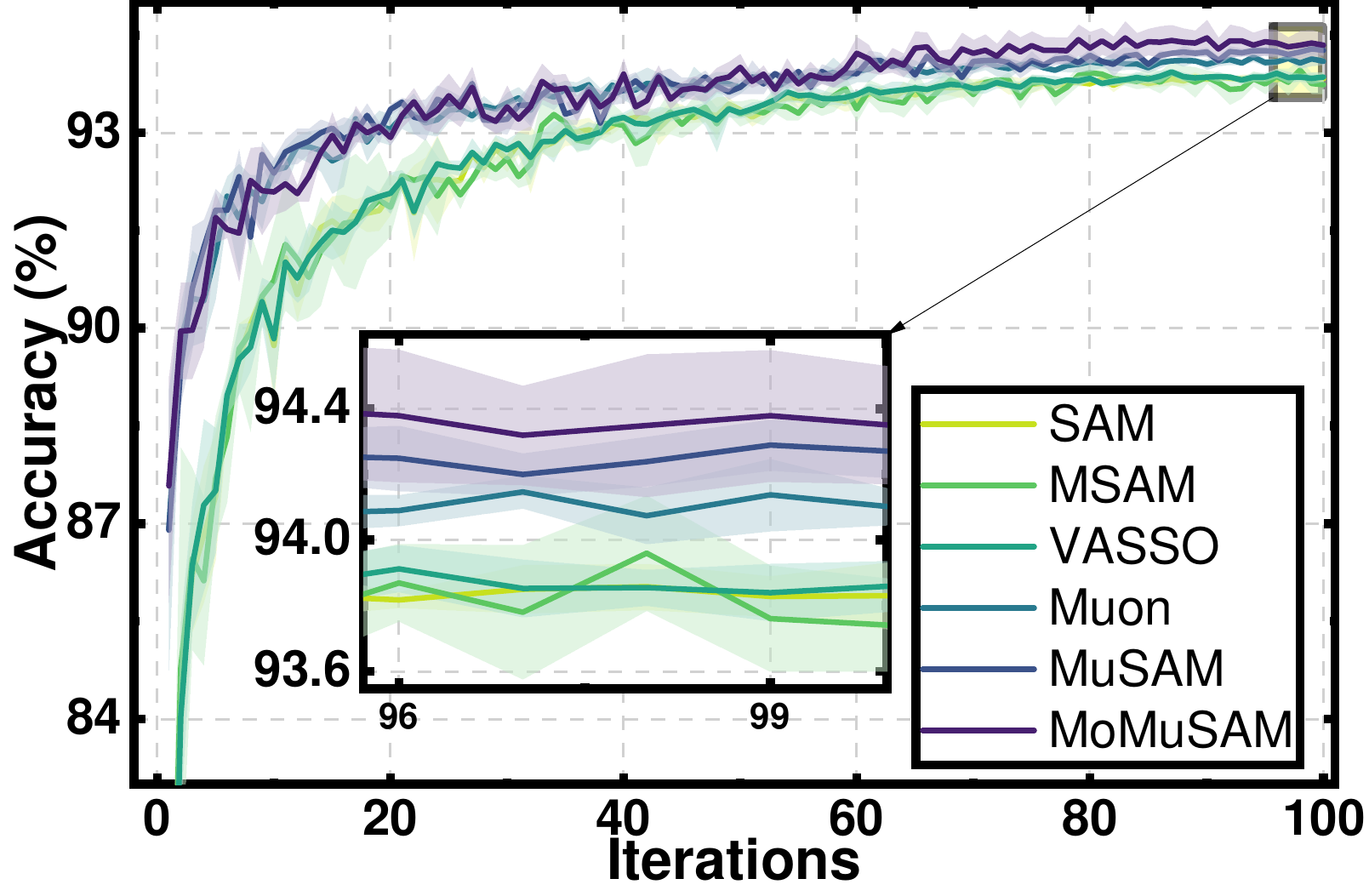}
    \par\vspace{0.1cm}
    \makebox[\textwidth]{{(a) FMNIST}}
\end{minipage}
\begin{minipage}[b]{0.2448\textwidth}
    \includegraphics[width=\textwidth]{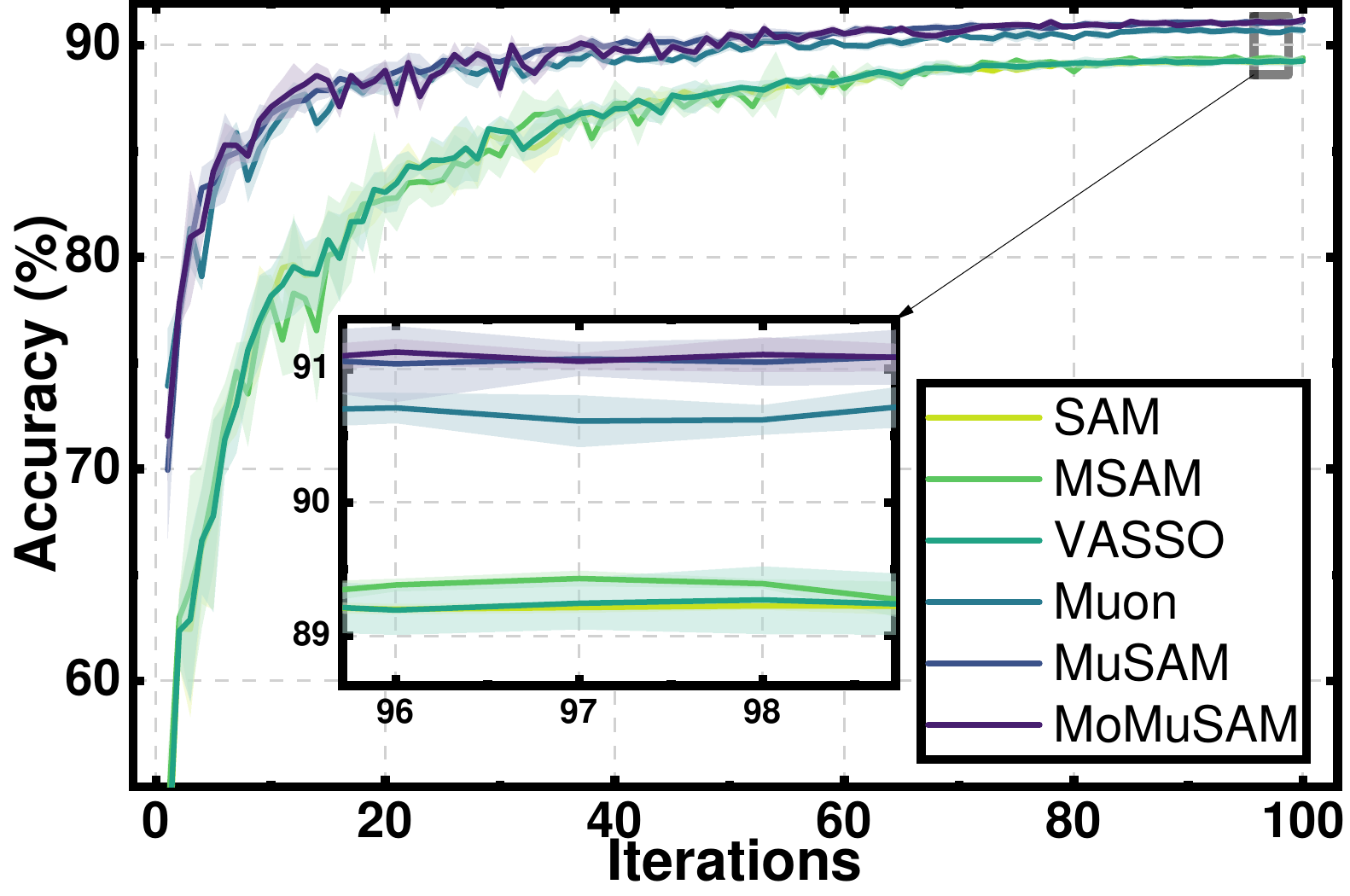}
    \par\vspace{0.1cm}
    \makebox[\textwidth]{{(b) CIFAR-10}}
\end{minipage}
\begin{minipage}[b]{0.2448\textwidth}
    \includegraphics[width=\textwidth]{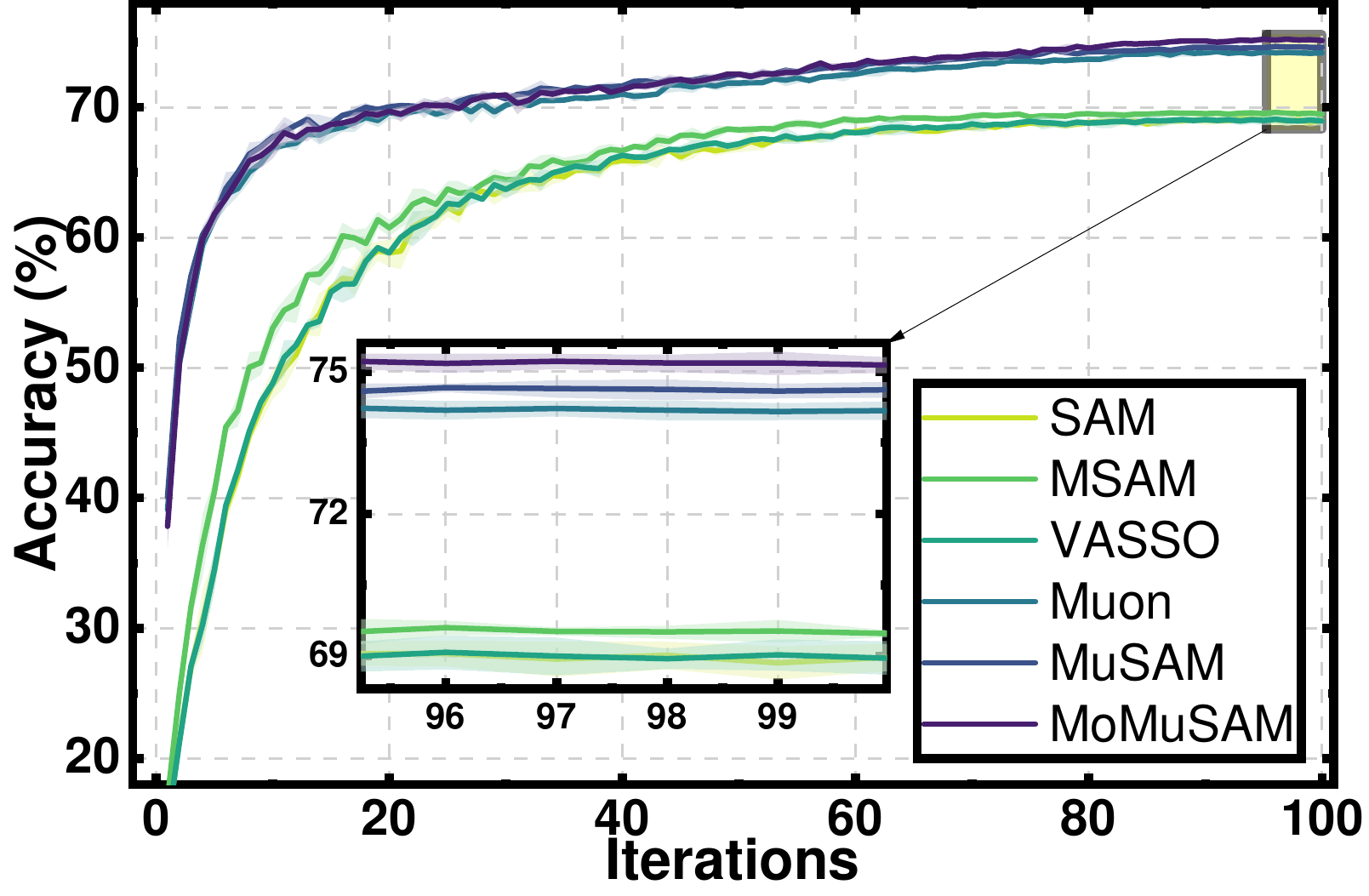}
    \par\vspace{0.1cm}
    \makebox[\textwidth]{{(c) CIFAR-100}}
\end{minipage}
\begin{minipage}[b]{0.2448\textwidth}
    \includegraphics[width=\textwidth]{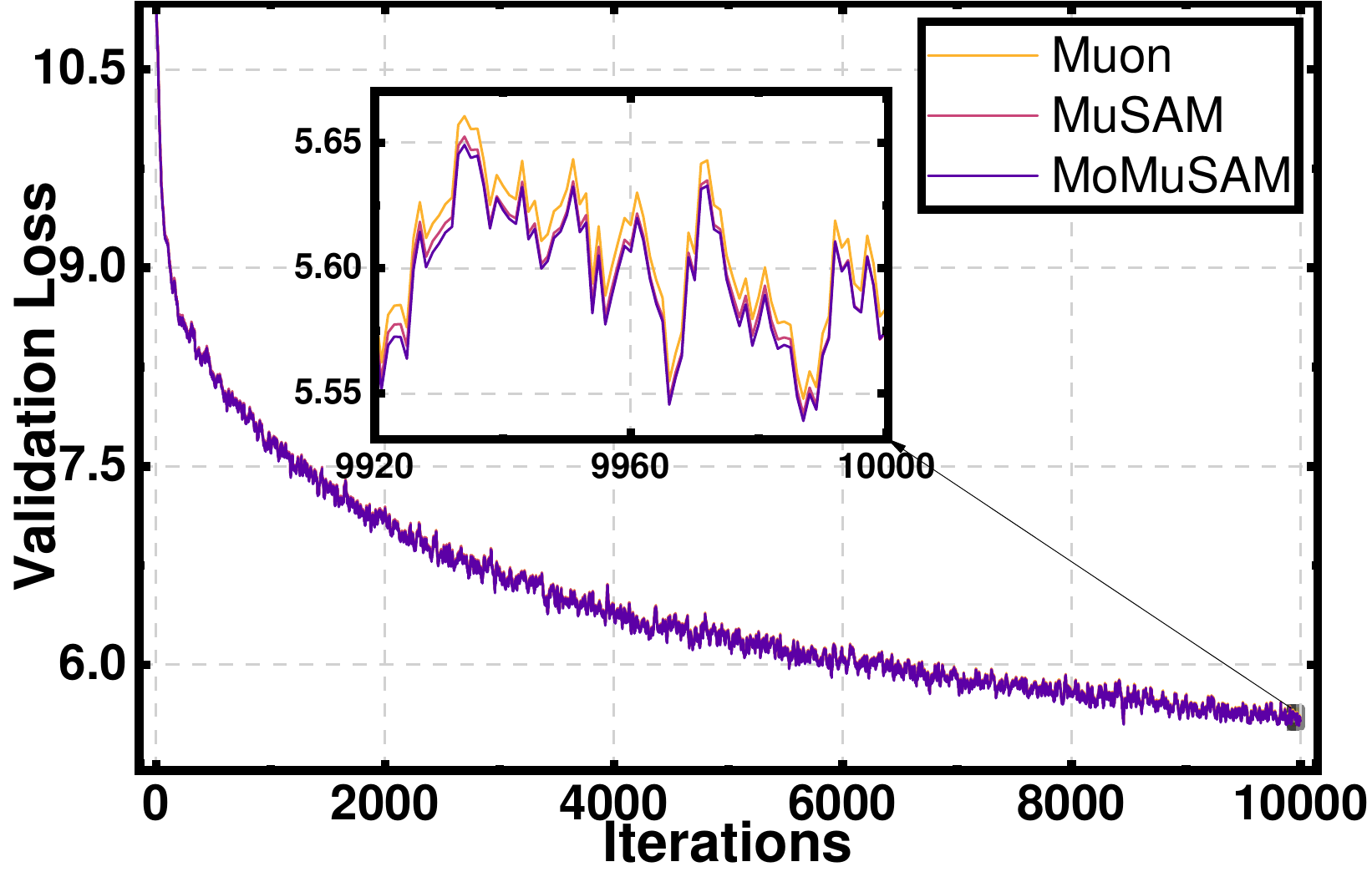}
    \par\vspace{0.1cm}
    \makebox[\textwidth]{{(d) OpenWebText}}
\end{minipage}
\caption{Performance comparison of different baselines across multiple datasets, showing test accuracy on FMNIST, CIFAR-10, and CIFAR-100, as well as validation loss on OpenWebText.}
\label{fig:accuracy4algorithm}
\end{figure*}

\begin{figure*}[t]
\centering
\begin{minipage}[t]{0.496\textwidth}
    \centering
    \captionsetup{type=figure}
    \setcounter{subfigure}{0} 
    \begin{subfigure}[t]{0.493\linewidth}
        \centering
        \includegraphics[width=\linewidth]{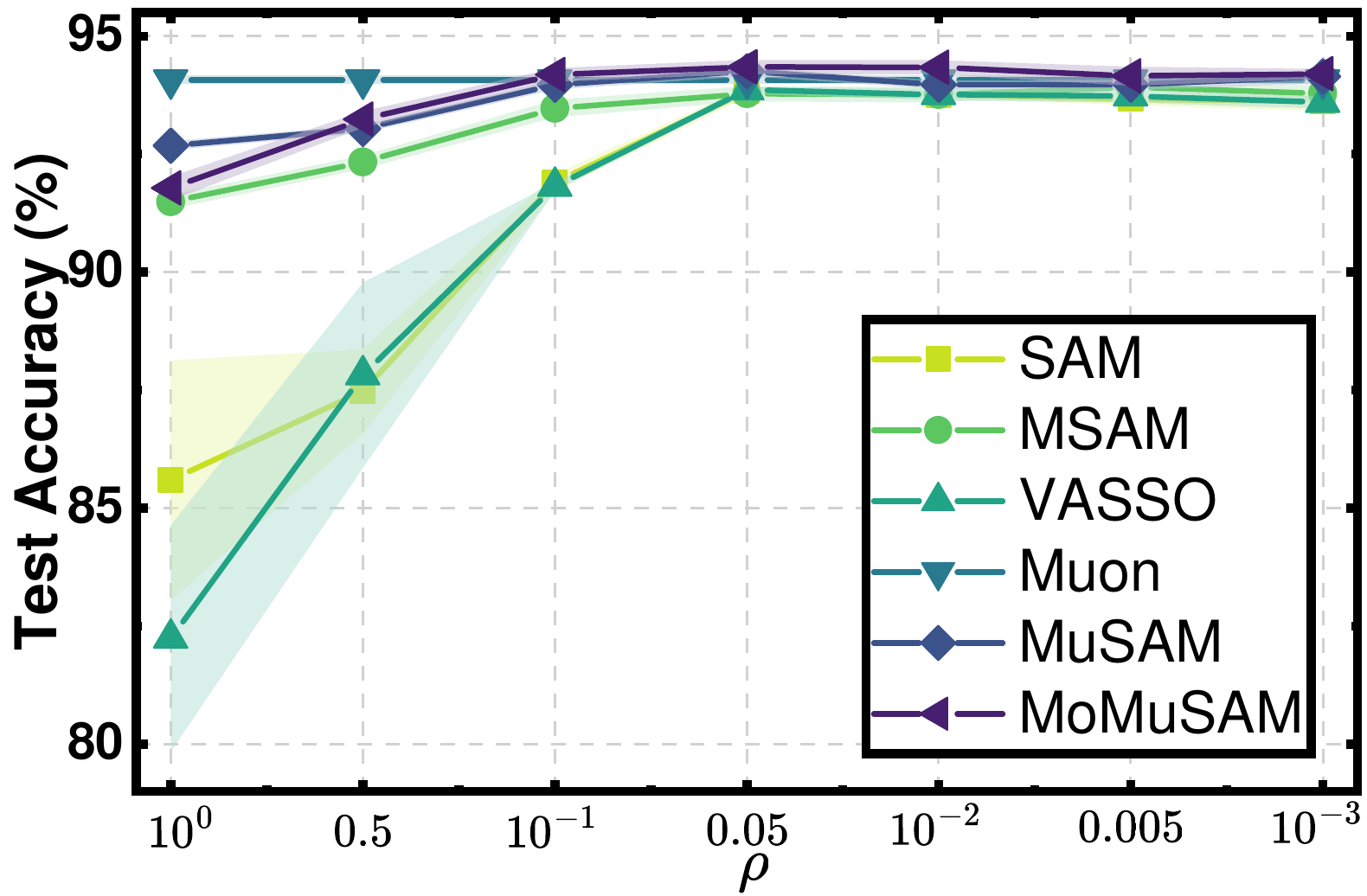}
        \caption{FMNIST}
        \label{fig:left-a}
    \end{subfigure}
    \hfill
    \begin{subfigure}[t]{0.493\linewidth}
        \centering
        \includegraphics[width=\linewidth]{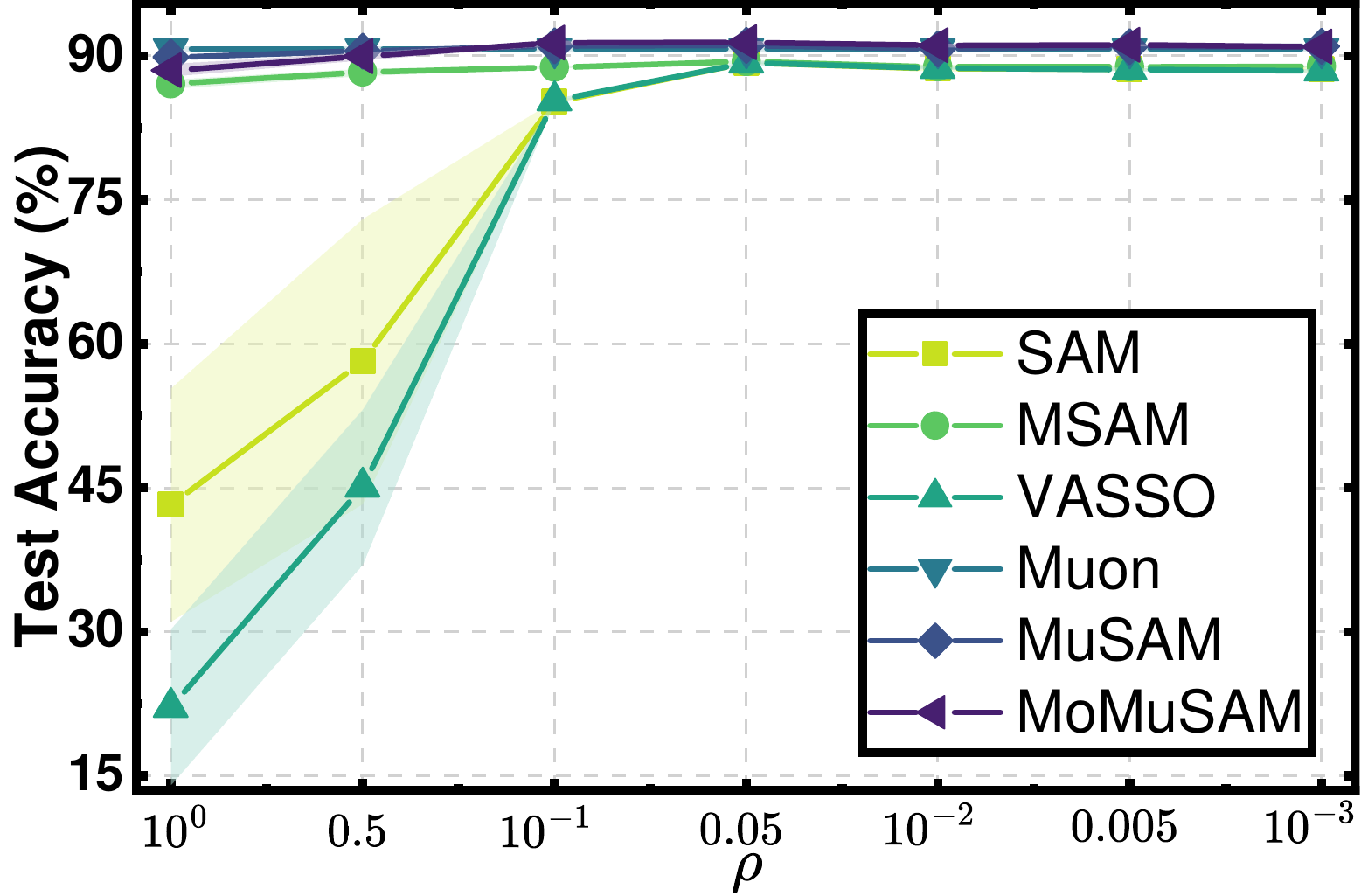}
        \caption{CIFAR-10}
        \label{fig:left-b}
    \end{subfigure}
    \caption{Sensitivity analysis on $\rho$.}
    \label{fig:rho}
\end{minipage}
\hfill
\begin{minipage}[t]{0.496\textwidth}
    \centering
    \captionsetup{type=figure}
    \setcounter{subfigure}{0} 

    \begin{subfigure}[t]{0.493\linewidth}
        \centering
        \includegraphics[width=\linewidth]{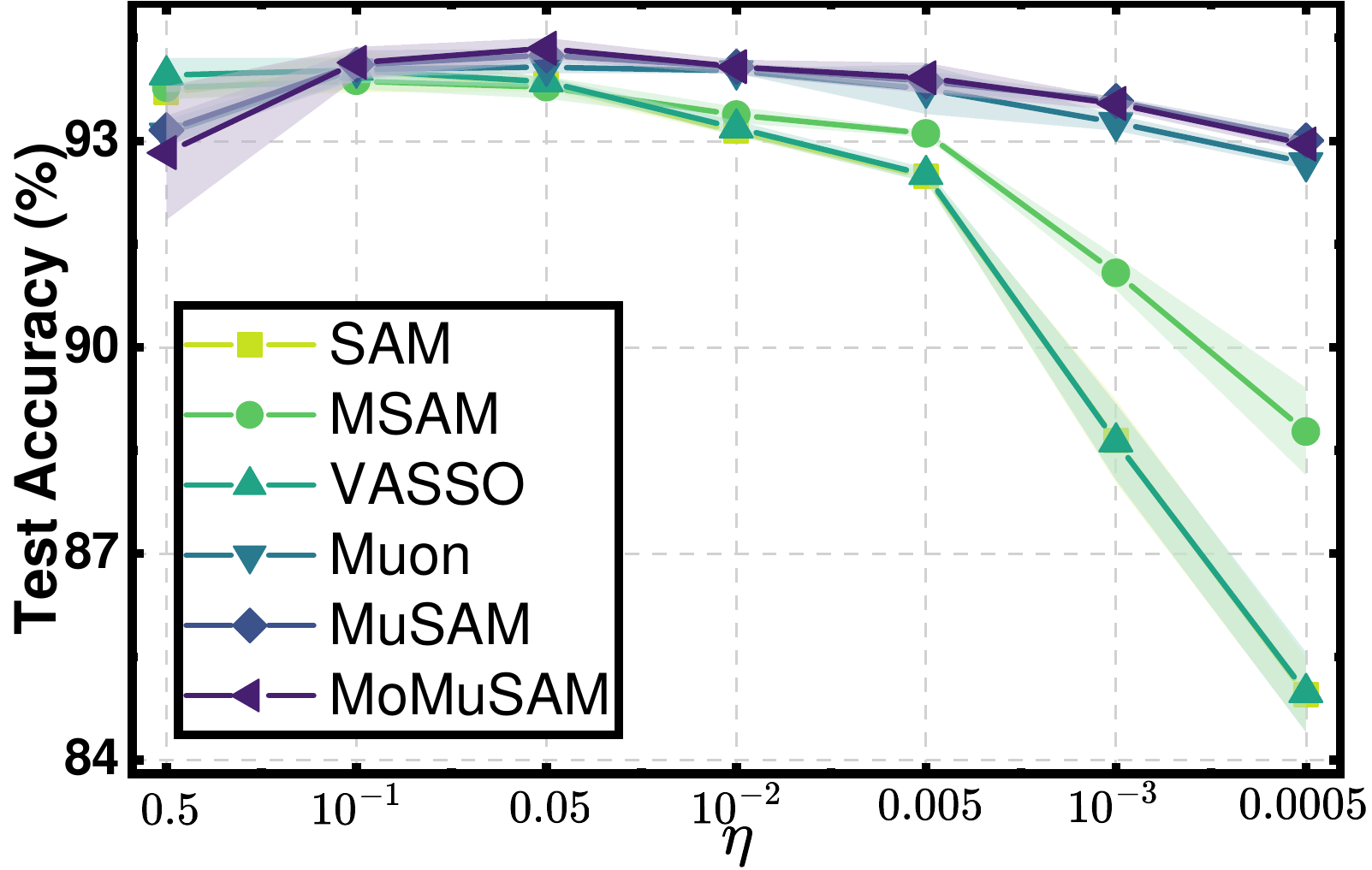}
        \caption{FMNIST}
        \label{fig:right-a}
    \end{subfigure}
    \hfill
    \begin{subfigure}[t]{0.493\linewidth}
        \centering
        \includegraphics[width=\linewidth]{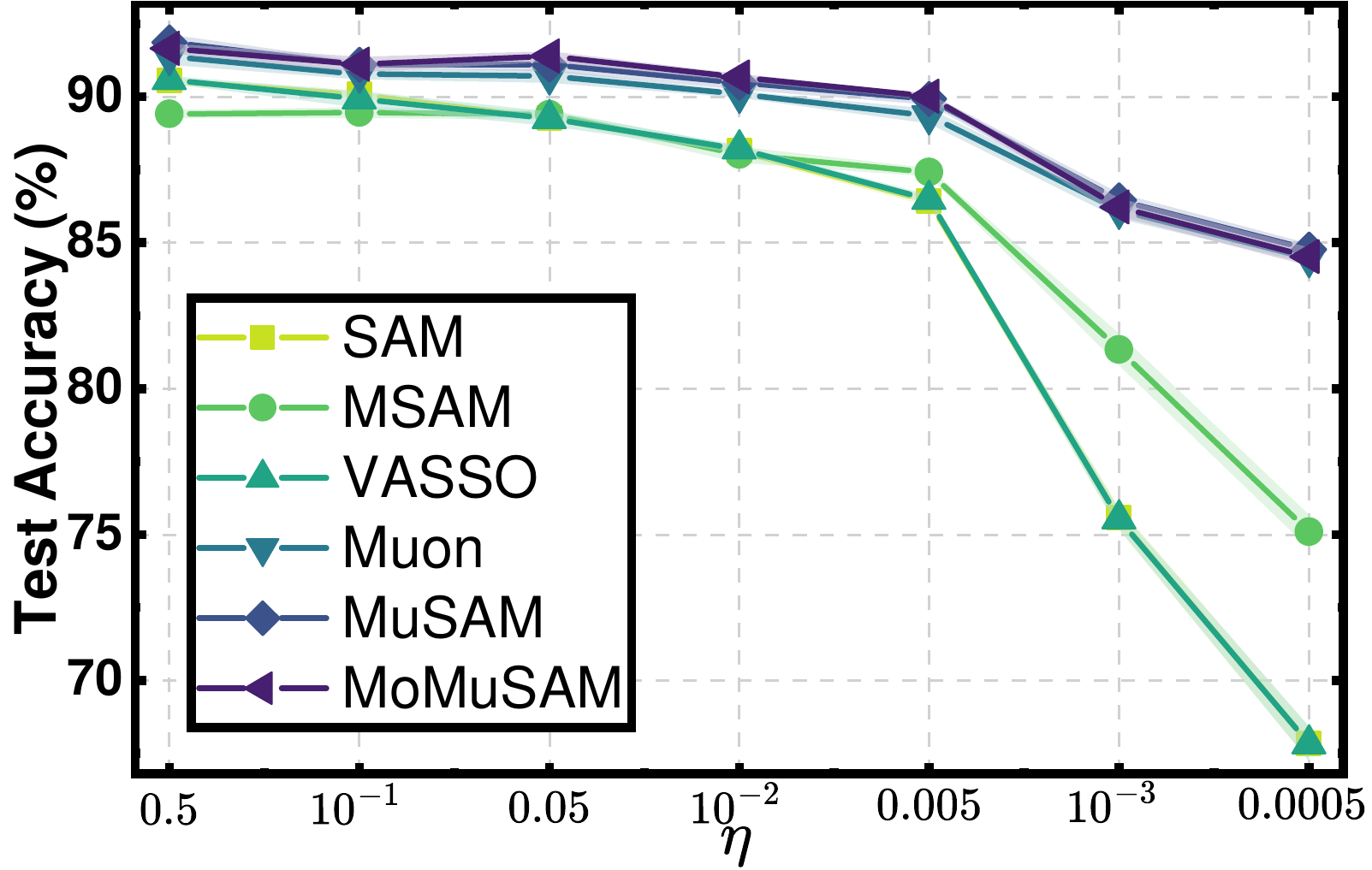}
        \caption{CIFAR-10}
        \label{fig:right-b}
    \end{subfigure}
    \caption{Sensitivity analysis on $\eta$.}
    \label{fig:stepsize}
\end{minipage}
\end{figure*}

\cref{fig:rho,fig:stepsize} examine the sensitivity of different optimizers to the perturbation radius \(\rho\) and learning rate \(\eta\). The results reveal a clear robustness advantage for our proposed methods. With respect to \(\rho\), \textsc{MuSAM} and \textsc{MoMuSAM} remain competitive over a substantially wider range of perturbation magnitudes than the SAM-style baselines. On both FMNIST and CIFAR-10, the performance of SAM, MSAM, and VASSO deteriorates rapidly when \(\rho\) becomes too large, suggesting that their perturbations can easily become overly aggressive and destabilize training. By contrast, \textsc{MuSAM} and \textsc{MoMuSAM} maintain strong accuracy across a much broader range of \(\rho\), with \textsc{MoMuSAM} consistently achieving the best or near-best performance. This indicates that the matrix-aware Muon update already provides a stabilizing optimization geometry, and that our sharpness-aware methods preserve this robustness while further enhancing generalization. A similar pattern is observed for the learning rate in \cref{fig:stepsize}. As \(\eta\) decreases, SAM, MSAM, and VASSO suffer much more pronounced degradation, whereas \textsc{MuSAM} and \textsc{MoMuSAM} remain comparatively stable and continue to deliver the strongest final accuracy over a broad range of stepsizes. This robustness is especially meaningful because it shows that our gains do not rely on fragile hyperparameter tuning. Overall, these sensitivity results show that the proposed methods are not only stronger in peak accuracy, but also substantially easier to tune in practice. Additional experimental results and setups are provided in \cref{experimentsec}.

\section{Conclusion}
In this paper, we introduced the first sharpness-aware formulation of Muon and established a new optimizer family that unifies Muon's matrix-aware orthogonalized updates with explicit local geometry control. Our proposed \textsc{MuSAM} incorporates explicit local geometry control into Muon's matrix-aware updates, while \textsc{MoMuSAM} further turns momentum into a sharpness-probing signal, removing the extra clean-gradient computation and stabilizing the perturbation direction. Theoretically, both methods attain the same asymptotic first-order stationarity order, while \textsc{MoMuSAM} yields a more stable stochastic linearization. Empirically, sharpness awareness consistently strengthens Muon, and \textsc{MoMuSAM} delivers a markedly better generalization--efficiency tradeoff.

\printbibliography

@article{hochreiter1994simplifying,
  title={Simplifying neural nets by discovering flat minima},
  author={Hochreiter, Sepp and Schmidhuber, J{\"u}rgen},
  journal={Advances in neural information processing systems},
  volume={7},
  year={1994}
}

@article{keskar2016large,
  title={On large-batch training for deep learning: Generalization gap and sharp minima},
  author={Keskar, Nitish Shirish and Mudigere, Dheevatsa and Nocedal, Jorge and Smelyanskiy, Mikhail and Tang, Ping Tak Peter},
  journal={arXiv preprint arXiv:1609.04836},
  year={2016}
}

@article{foret2020sharpness,
  title={Sharpness-aware minimization for efficiently improving generalization},
  author={Foret, Pierre and Kleiner, Ariel and Mobahi, Hossein and Neyshabur, Behnam},
  journal={arXiv preprint arXiv:2010.01412},
  year={2020}
}

@inproceedings{andriushchenko2022towards,
  title={Towards understanding sharpness-aware minimization},
  author={Andriushchenko, Maksym and Flammarion, Nicolas},
  booktitle={International conference on machine learning},
  pages={639--668},
  year={2022},
  organization={PMLR}
}

@article{wen2022does,
  title={How does sharpness-aware minimization minimize sharpness?},
  author={Wen, Kaiyue and Ma, Tengyu and Li, Zhiyuan},
  journal={arXiv preprint arXiv:2211.05729},
  year={2022}
}

@article{bartlett2023dynamics,
  title={The dynamics of sharpness-aware minimization: Bouncing across ravines and drifting towards wide minima},
  author={Bartlett, Peter L and Long, Philip M and Bousquet, Olivier},
  journal={Journal of Machine Learning Research},
  volume={24},
  number={316},
  pages={1--36},
  year={2023}
}

@article{becker2024momentum,
  title={Momentum-sam: Sharpness aware minimization without computational overhead},
  author={Becker, Marlon and Altrock, Frederick and Risse, Benjamin},
  journal={arXiv preprint arXiv:2401.12033},
  year={2024}
}

@online{Muon2024,
  author = {Jordan, Keller and Jin, Yiding and Boza, Victor and Jiacheng, Yang and Cecista, Franciszek and Newhouse, Landon and Bernstein, Jeremy},
  title = {Muon: An Optimizer for Hidden Layers in Neural Networks},
  year = {2024},
  url = {https://kellerjordan.github.io/posts/muon}
}

@article{wen2018smoothout,
  title={Smoothout: Smoothing out sharp minima to improve generalization in deep learning},
  author={Wen, Wei and Wang, Yandan and Yan, Feng and Xu, Cong and Wu, Chunpeng and Chen, Yiran and Li, Hai},
  journal={arXiv preprint arXiv:1805.07898},
  year={2018}
}

@article{zhuang2022surrogate,
  title={Surrogate gap minimization improves sharpness-aware training},
  author={Zhuang, Juntang and Gong, Boqing and Yuan, Liangzhe and Cui, Yin and Adam, Hartwig and Dvornek, Nicha and Tatikonda, Sekhar and Duncan, James and Liu, Ting},
  journal={arXiv preprint arXiv:2203.08065},
  year={2022}
}

@article{zhao2022stochastic,
  title={Stochastic scheduled sharpness-aware minimization for efficiently training deep neural networks},
  author={Zhao, Yang and Zhang, Hao and Ss-sam, Xiuyuan Hu},
  journal={arXiv preprint arXiv:2203.09962},
  volume={8},
  pages={18},
  year={2022}
}

@inproceedings{reddi2016stochastic,
  title={Stochastic Frank-Wolfe methods for nonconvex optimization},
  author={Reddi, Sashank J and Sra, Suvrit and P{\'o}czos, Barnab{\'a}s and Smola, Alex},
  booktitle={2016 54th annual Allerton conference on communication, control, and computing (Allerton)},
  pages={1244--1251},
  year={2016},
  organization={IEEE}
}

@article{li2025vasso,
  title={Vasso: Variance suppression for sharpness-aware minimization},
  author={Li, Bingcong and Zhang, Yilang and Giannakis, Georgios B},
  journal={arXiv preprint arXiv:2509.02433},
  year={2025}
}

@article{li2021heavy,
  title={Heavy ball momentum for conditional gradient},
  author={Li, Bingcong and Sadeghi, Alireza and Giannakis, Georgios},
  journal={Advances in Neural Information Processing Systems},
  volume={34},
  pages={21244--21255},
  year={2021}
}

@article{cutkosky2019momentum,
  title={Momentum-based variance reduction in non-convex sgd},
  author={Cutkosky, Ashok and Orabona, Francesco},
  journal={Advances in neural information processing systems},
  volume={32},
  year={2019}
}

@article{qian2025muon,
  title={Muon is Provably Faster with Momentum Variance Reduction},
  author={Qian, Xun and Rammal, Hussein and Kovalev, Dmitry and Richtarik, Peter},
  journal={arXiv preprint arXiv:2512.16598},
  year={2025}
}

@article{ghadimi2013stochastic,
  title={Stochastic first-and zeroth-order methods for nonconvex stochastic programming},
  author={Ghadimi, Saeed and Lan, Guanghui},
  journal={SIAM journal on optimization},
  volume={23},
  number={4},
  pages={2341--2368},
  year={2013},
  publisher={SIAM}
}

@article{bottou2018optimization,
  title={Optimization methods for large-scale machine learning},
  author={Bottou, L{\'e}on and Curtis, Frank E and Nocedal, Jorge},
  journal={SIAM review},
  volume={60},
  number={2},
  pages={223--311},
  year={2018},
  publisher={SIAM}
}

@article{mi2022make,
  title={Make sharpness-aware minimization stronger: A sparsified perturbation approach},
  author={Mi, Peng and Shen, Li and Ren, Tianhe and Zhou, Yiyi and Sun, Xiaoshuai and Ji, Rongrong and Tao, Dacheng},
  journal={Advances in Neural Information Processing Systems},
  volume={35},
  pages={30950--30962},
  year={2022}
}

@article{kim2026convergence,
  title={Convergence of Muon with Newton-Schulz},
  author={Kim, Gyu Yeol and Oh, Min-hwan},
  journal={arXiv preprint arXiv:2601.19156},
  year={2026}
}

@book{nesterov2013introductory,
  title={Introductory lectures on convex optimization: A basic course},
  author={Nesterov, Yurii},
  volume={87},
  year={2013},
  publisher={Springer Science \& Business Media}
}

@inproceedings{jaggi2013revisiting,
  title={Revisiting Frank-Wolfe: Projection-free sparse convex optimization},
  author={Jaggi, Martin},
  booktitle={International conference on machine learning},
  pages={427--435},
  year={2013},
  organization={PMLR}
}

@book{beck2017first,
  title={First-order methods in optimization},
  author={Beck, Amir},
  year={2017},
  publisher={SIAM}
}

@article{liu2025muon,
  title={Muon is scalable for llm training},
  author={Liu, Jingyuan and Su, Jianlin and Yao, Xingcheng and Jiang, Zhejun and Lai, Guokun and Du, Yulun and Qin, Yidao and Xu, Weixin and Lu, Enzhe and Yan, Junjie and others},
  journal={arXiv preprint arXiv:2502.16982},
  year={2025}
}

@article{an2025asgo,
  title={Asgo: Adaptive structured gradient optimization},
  author={An, Kang and Liu, Yuxing and Pan, Rui and Ren, Yi and Ma, Shiqian and Goldfarb, Donald and Zhang, Tong},
  journal={arXiv preprint arXiv:2503.20762},
  year={2025}
}

@article{liu2025cosmos,
  title={Cosmos: A hybrid adaptive optimizer for memory-efficient training of llms},
  author={Liu, Liming and Xu, Zhenghao and Zhang, Zixuan and Kang, Hao and Li, Zichong and Liang, Chen and Chen, Weizhu and Zhao, Tuo},
  journal={arXiv preprint arXiv:2502.17410},
  year={2025}
}

@article{shah2025practical,
  title={Practical efficiency of muon for pretraining},
  author={Shah, Ishaan and Polloreno, Anthony M and Stratos, Karl and Monk, Philip and Chaluvaraju, Adarsh and Hojel, Andrew and Ma, Andrew and Thomas, Anil and Tanwer, Ashish and Shah, Darsh J and others},
  journal={arXiv preprint arXiv:2505.02222},
  year={2025}
}

@article{si2025adamuon,
  title={Adamuon: Adaptive muon optimizer},
  author={Si, Chongjie and Zhang, Debing and Shen, Wei},
  journal={arXiv preprint arXiv:2507.11005},
  year={2025}
}

@article{lau2025polargrad,
  title={Polargrad: A class of matrix-gradient optimizers from a unifying preconditioning perspective},
  author={Lau, Tim Tsz-Kit and Long, Qi and Su, Weijie},
  journal={arXiv preprint arXiv:2505.21799},
  year={2025}
}

@article{zhang2025adagrad,
  title={Adagrad meets muon: Adaptive stepsizes for orthogonal updates},
  author={Zhang, Minxin and Liu, Yuxuan and Schaeffer, Hayden},
  journal={arXiv preprint arXiv:2509.02981},
  year={2025}
}

@article{riabinin2025gluon,
  title={Gluon: Making muon \& scion great again!(bridging theory and practice of lmo-based optimizers for llms)},
  author={Riabinin, Artem and Shulgin, Egor and Gruntkowska, Kaja and Richt{\'a}rik, Peter},
  journal={arXiv preprint arXiv:2505.13416},
  year={2025}
}

@article{pethick2025training,
  title={Training deep learning models with norm-constrained lmos},
  author={Pethick, Thomas and Xie, Wanyun and Antonakopoulos, Kimon and Zhu, Zhenyu and Silveti-Falls, Antonio and Cevher, Volkan},
  journal={arXiv preprint arXiv:2502.07529},
  year={2025}
}

@article{bernstein2024old,
  title={Old optimizer, new norm: An anthology},
  author={Bernstein, Jeremy and Newhouse, Laker},
  journal={arXiv preprint arXiv:2409.20325},
  year={2024}
}

@article{kovalev2025understanding,
  title={Understanding gradient orthogonalization for deep learning via non-euclidean trust-region optimization},
  author={Kovalev, Dmitry},
  journal={arXiv preprint arXiv:2503.12645},
  year={2025}
}

@article{sfyraki2025lions,
  title={Lions and muons: Optimization via stochastic frank-wolfe},
  author={Sfyraki, Maria-Eleni and Wang, Jun-Kun},
  journal={arXiv preprint arXiv:2506.04192},
  year={2025}
}

@article{chen2025muon,
  title={Muon optimizes under spectral norm constraints},
  author={Chen, Lizhang and Li, Jonathan and Liu, Qiang},
  journal={arXiv preprint arXiv:2506.15054},
  year={2025}
}

@article{fan2025implicit,
  title={Implicit bias of spectral descent and muon on multiclass separable data},
  author={Fan, Chen and Schmidt, Mark and Thrampoulidis, Christos},
  journal={arXiv preprint arXiv:2502.04664},
  year={2025}
}

@article{shen2025convergence,
  title={On the convergence analysis of muon},
  author={Shen, Wei and Huang, Ruichuan and Huang, Minhui and Shen, Cong and Zhang, Jiawei},
  journal={arXiv preprint arXiv:2505.23737},
  year={2025}
}

@article{sato2025convergence,
  title={Convergence bound and critical batch size of muon optimizer},
  author={Sato, Naoki and Naganuma, Hiroki and Iiduka, Hideaki},
  journal={arXiv preprint arXiv:2507.01598},
  year={2025}
}

@article{li2025note,
  title={A note on the convergence of muon},
  author={Li, Jiaxiang and Hong, Mingyi},
  journal={arXiv preprint arXiv:2502.02900},
  year={2025}
}

@article{chang2025convergence,
  title={On the Convergence of Muon and Beyond},
  author={Chang, Da and Liu, Yongxiang and Yuan, Ganzhao},
  journal={arXiv preprint arXiv:2509.15816},
  year={2025}
}

@article{li2018visualizing,
  title={Visualizing the loss landscape of neural nets},
  author={Li, Hao and Xu, Zheng and Taylor, Gavin and Studer, Christoph and Goldstein, Tom},
  journal={Advances in neural information processing systems},
  volume={31},
  year={2018}
}

@article{yao2018hessian,
  title={Hessian-based analysis of large batch training and robustness to adversaries},
  author={Yao, Zhewei and Gholami, Amir and Lei, Qi and Keutzer, Kurt and Mahoney, Michael W},
  journal={Advances in Neural Information Processing Systems},
  volume={31},
  year={2018}
}

@article{liu2020loss,
  title={On the loss landscape of adversarial training: Identifying challenges and how to overcome them},
  author={Liu, Chen and Salzmann, Mathieu and Lin, Tao and Tomioka, Ryota and S{\"u}sstrunk, Sabine},
  journal={Advances in Neural Information Processing Systems},
  volume={33},
  pages={21476--21487},
  year={2020}
}

@article{damian2021label,
  title={Label noise sgd provably prefers flat global minimizers},
  author={Damian, Alex and Ma, Tengyu and Lee, Jason D},
  journal={Advances in Neural Information Processing Systems},
  volume={34},
  pages={27449--27461},
  year={2021}
}

@inproceedings{dinh2017sharp,
  title={Sharp minima can generalize for deep nets},
  author={Dinh, Laurent and Pascanu, Razvan and Bengio, Samy and Bengio, Yoshua},
  booktitle={International Conference on Machine Learning},
  pages={1019--1028},
  year={2017},
  organization={PMLR}
}

@article{chaudhari2019entropy,
  title={Entropy-sgd: Biasing gradient descent into wide valleys},
  author={Chaudhari, Pratik and Choromanska, Anna and Soatto, Stefano and LeCun, Yann and Baldassi, Carlo and Borgs, Christian and Chayes, Jennifer and Sagun, Levent and Zecchina, Riccardo},
  journal={Journal of Statistical Mechanics: Theory and Experiment},
  volume={2019},
  number={12},
  pages={124018},
  year={2019},
  publisher={IOP Publishing and SISSA}
}

@article{wu2020adversarial,
  title={Adversarial weight perturbation helps robust generalization},
  author={Wu, Dongxian and Xia, Shu-Tao and Wang, Yisen},
  journal={Advances in neural information processing systems},
  volume={33},
  pages={2958--2969},
  year={2020}
}

@inproceedings{zheng2021regularizing,
  title={Regularizing neural networks via adversarial model perturbation},
  author={Zheng, Yaowei and Zhang, Richong and Mao, Yongyi},
  booktitle={Proceedings of the IEEE/CVF Conference on Computer Vision and Pattern Recognition},
  pages={8156--8165},
  year={2021}
}

@inproceedings{kwon2021asam,
  title={Asam: Adaptive sharpness-aware minimization for scale-invariant learning of deep neural networks},
  author={Kwon, Jungmin and Kim, Jeongseop and Park, Hyunseo and Choi, In Kwon},
  booktitle={International conference on machine learning},
  pages={5905--5914},
  year={2021},
  organization={PMLR}
}

@inproceedings{kim2022fisher,
  title={Fisher sam: Information geometry and sharpness aware minimisation},
  author={Kim, Minyoung and Li, Da and Hu, Shell X and Hospedales, Timothy},
  booktitle={International Conference on Machine Learning},
  pages={11148--11161},
  year={2022},
  organization={PMLR}
}

@inproceedings{liu2022towards,
  title={Towards efficient and scalable sharpness-aware minimization},
  author={Liu, Yong and Mai, Siqi and Chen, Xiangning and Hsieh, Cho-Jui and You, Yang},
  booktitle={Proceedings of the IEEE/CVF Conference on Computer Vision and Pattern Recognition},
  pages={12360--12370},
  year={2022}
}

@article{jiang2023adaptive,
  title={An adaptive policy to employ sharpness-aware minimization},
  author={Jiang, Weisen and Yang, Hansi and Zhang, Yu and Kwok, James},
  journal={arXiv preprint arXiv:2304.14647},
  year={2023}
}

@article{du2022sharpness,
  title={Sharpness-aware training for free},
  author={Du, Jiawei and Zhou, Daquan and Feng, Jiashi and Tan, Vincent and Zhou, Joey Tianyi},
  journal={Advances in Neural Information Processing Systems},
  volume={35},
  pages={23439--23451},
  year={2022}
}

@misc{karpathy2022nanoGPT,
  author       = {Andrej Karpathy},
  title        = {nanoGPT},
  year         = {2022},
  url          = {https://github.com/karpathy/nanoGPT}
}

@article{radford2019language,
  title={Language models are unsupervised multitask learners},
  author={Radford, Alec and Wu, Jeffrey and Child, Rewon and Luan, David and Amodei, Dario and Sutskever, Ilya and others},
  journal={OpenAI blog},
  volume={1},
  number={8},
  pages={9},
  year={2019}
}

@misc{Gokaslan2019OpenWeb,  
	title={OpenWebText Corpus},
	author={Aaron Gokaslan and Vanya Cohen},
	howpublished={\url{http://Skylion007.github.io/OpenWebTextCorpus}}, 
	year={2019}
}

@article{xiao2017fashion,
  title={Fashion-mnist: a novel image dataset for benchmarking machine learning algorithms},
  author={Xiao, Han and Rasul, Kashif and Vollgraf, Roland},
  journal={arXiv preprint arXiv:1708.07747},
  year={2017}
}

@article{krizhevsky2009learning,
  title={Learning multiple layers of features from tiny images},
  author={Krizhevsky, Alex and Hinton, Geoffrey and others},
  year={2009},
  publisher={Toronto, ON, Canada}
}

@article{vyas2024soap,
  title={Soap: Improving and stabilizing shampoo using adam},
  author={Vyas, Nikhil and Morwani, Depen and Zhao, Rosie and Kwun, Mujin and Shapira, Itai and Brandfonbrener, David and Janson, Lucas and Kakade, Sham},
  journal={arXiv preprint arXiv:2409.11321},
  year={2024}
}

@article{chen2023symbolic,
  title={Symbolic discovery of optimization algorithms},
  author={Chen, Xiangning and Liang, Chen and Huang, Da and Real, Esteban and Wang, Kaiyuan and Pham, Hieu and Dong, Xuanyi and Luong, Thang and Hsieh, Cho-Jui and Lu, Yifeng and others},
  journal={Advances in neural information processing systems},
  volume={36},
  pages={49205--49233},
  year={2023}
}

@article{yuan2024mars,
  title={Mars: Unleashing the power of variance reduction for training large models},
  author={Yuan, Huizhuo and Liu, Yifeng and Wu, Shuang and Zhou, Xun and Gu, Quanquan},
  journal={arXiv preprint arXiv:2411.10438},
  year={2024}
}

\newpage
\appendix
\section{Related Works}
\paragraph{Muon and matrix-aware orthogonalized optimization.}
Recent work has increasingly explored optimizers that exploit the intrinsic matrix structure of neural network parameters rather than treating each layer as a long vector. Within this line, Muon \citep{Muon2024} is a representative matrix-aware optimizer that orthogonalizes the momentum matrix before updating the weights, yielding structured search directions with strong empirical performance, especially at large scale \citep{liu2025muon, shah2025practical}. In practice, Muon replaces exact SVD-based polar decomposition with a few Newton--Schulz iterations, resulting in an efficient SVD-free implementation based primarily on matrix multiplications \citep{kim2026convergence}. Around Muon, a growing ecosystem of related methods has emerged, including AdaMuon, COSMOS, ASGO, PolarGrad, AdaGO, Gluon, and other matrix-aware or LMO-based optimizers \citep{si2025adamuon, liu2025cosmos, an2025asgo, lau2025polargrad, zhang2025adagrad, riabinin2025gluon, pethick2025training}. On the theoretical side, Muon has been interpreted through non-Euclidean trust-region geometry, stochastic Frank--Wolfe viewpoints, spectral-norm constrained optimization, and implicit-bias analyses \citep{bernstein2024old, kovalev2025understanding, sfyraki2025lions, chen2025muon, fan2025implicit}. Nevertheless, existing Muon literature remains centered on optimization geometry, convergence, and efficiency, while exposing a persistent gap between practical finite-step Newton--Schulz implementations and analyses based on exact polar factors, as well as relatively slow or restrictive guarantees for several Muon variants \citep{shen2025convergence, sato2025convergence, li2025note, chang2025convergence}. Crucially for our setting, these works do not establish any explicit mechanism by which Muon controls local sharpness or steers optimization toward flat regions: Muon reshapes the update direction, but any generalization benefit from flatness can only arise implicitly. This makes sharpness-aware Muon especially compelling, because Muon's matrix-aware inductive bias shapes update geometry, while explicit sharpness control governs the local loss geometry most relevant to generalization.

\paragraph{Sharpness-aware optimization.}
Sharpness-aware optimization is motivated by the long-standing observation that flatter minima often correlate with stronger generalization, although this connection is subtle and not universal \citep{hochreiter1994simplifying, keskar2016large, li2018visualizing, yao2018hessian, liu2020loss, damian2021label, dinh2017sharp}. Early efforts such as Entropy-SGD and adversarial weight perturbation already sought to bias optimization toward flatter regions \citep{chaudhari2019entropy, wu2020adversarial, zheng2021regularizing}, but SAM has become the dominant formulation by explicitly minimizing a local worst-case loss through adversarial perturbations in parameter space \citep{foret2020sharpness}. Its empirical success has motivated a broad family of variants, including ASAM, Fisher-SAM, GSAM, and a range of efficiency-oriented methods such as LookSAM, SS-SAM, adaptive SAM policies, and SAF \citep{kwon2021asam, kim2022fisher, zhuang2022surrogate, liu2022towards, zhao2022stochastic, jiang2023adaptive, du2022sharpness}. At the same time, this line of work has exposed two persistent limitations. First, vanilla SAM typically requires an additional forward--backward pass to construct the perturbation, which substantially increases training cost \citep{foret2020sharpness, wen2022does, becker2024momentum}. Second, the perturbation is produced by stochastic linearization from a single minibatch gradient, making it a batch-local and noisy proxy for sharpness; recent work shows that such adversaries can become unstable or overly friendly, thereby weakening the intended sharpness signal and limiting the generalization--computation tradeoff \citep{andriushchenko2022towards, becker2024momentum, li2025vasso}. These limitations are especially relevant in our setting: once sharpness-aware training is brought into Muon, the key challenge is no longer only how to inject local geometry control, but how to do so in a way that respects Muon's structured update geometry while avoiding the overhead and instability of single-batch perturbations. This is precisely the gap addressed by our momentum-guided sharpness-aware Muon design.

\section{Proof}
\subsection{Useful lemmas}
\begin{lemma}\label{lemma:sam_descent}
Under \cref{ass:lower_bound,lipschiz,gradientvariance}, the Muon+SAM update guarantees that
\begin{align}
f(W_t)
&\le
f(W_{t-1})
-\eta\|\nabla f(W_{t-1})\|_*
+2\eta\|M_t-\nabla f(W_{t-1}+\delta_t)\|_*
+2\eta L\rho
+\frac{L\eta^2}{2}.
\label{eq:sam_descent_final}
\end{align}
\end{lemma}

\begin{proof}
Let the exact SVD of $M_t$ be $M_t = Q_t\Sigma_tR_t^\top$,
and define $U_t := \mathrm{Polar}(M_t)=Q_tR_t^\top$.
By the descent lemma implied by \cref{lipschiz}, applied to
\(W_t = W_{t-1}-\eta U_t\), we have
\begin{align}
f(W_t)
&\le
f(W_{t-1})
+\langle \nabla f(W_{t-1}),W_t-W_{t-1}\rangle_F
+\frac{L}{2}\|W_t-W_{t-1}\|_{\mathrm{op}}^2
\nonumber\\
&=
f(W_{t-1})
-\eta\langle \nabla f(W_{t-1}),U_t\rangle_F
+\frac{L\eta^2}{2}\|U_t\|_{\mathrm{op}}^2
\nonumber\\
&\le
f(W_{t-1})
-\eta\langle \nabla f(W_{t-1}),U_t\rangle_F
+\frac{L\eta^2}{2}.
\label{eq:sam_descent_start_main}
\end{align}
Next,
\begin{align}
\langle \nabla f(W_{t-1}),U_t\rangle_F
&=
\langle M_t,U_t\rangle_F
+\langle \nabla f(W_{t-1})-M_t,U_t\rangle_F
\nonumber\\
&=
\|M_t\|_*
+\langle \nabla f(W_{t-1})-M_t,U_t\rangle_F
\nonumber\\
&\ge
\|M_t\|_*-\|\nabla f(W_{t-1})-M_t\|_*\,\|U_t\|_{\mathrm{op}}
\nonumber\\
&\ge
\|M_t\|_*-\|\nabla f(W_{t-1})-M_t\|_*.
\label{eq:sam_inner_lb_1_main}
\end{align}
By the reverse triangle inequality for the nuclear norm,
\begin{equation}
\|M_t\|_*
\ge
\|\nabla f(W_{t-1})\|_*-\|\nabla f(W_{t-1})-M_t\|_*.
\label{eq:sam_reverse_triangle_main}
\end{equation}
Combining \cref{eq:sam_inner_lb_1_main} and \cref{eq:sam_reverse_triangle_main}, we obtain
\begin{equation}
\langle \nabla f(W_{t-1}),U_t\rangle_F
\ge
\|\nabla f(W_{t-1})\|_*
-2\|\nabla f(W_{t-1})-M_t\|_*.
\label{eq:sam_inner_lb_2_main}
\end{equation}
Substituting \cref{eq:sam_inner_lb_2_main} into \cref{eq:sam_descent_start_main} yields
\begin{align}
f(W_t)
&\le
f(W_{t-1})
-\eta\|\nabla f(W_{t-1})\|_*
+2\eta\|\nabla f(W_{t-1})-M_t\|_*
+\frac{L\eta^2}{2}.
\label{eq:sam_descent_mid_main}
\end{align}
Now add and subtract $\nabla f(W_{t-1}+\delta_t)$:
\begin{align}
\|\nabla f(W_{t-1})-M_t\|_*
&\le
\|\nabla f(W_{t-1})-\nabla f(W_{t-1}+\delta_t)\|_*
+\|\nabla f(W_{t-1}+\delta_t)-M_t\|_*.
\end{align}
By \cref{lipschiz},
\begin{align}
\|\nabla f(W_{t-1})-\nabla f(W_{t-1}+\delta_t)\|_*\le
L\|\delta_t\|_{\mathrm{op}}
\le
L\|\delta_t\|_F
\le
L\rho.
\end{align}
Therefore,
\begin{equation}
\|\nabla f(W_{t-1})-M_t\|_*
\le
\|M_t-\nabla f(W_{t-1}+\delta_t)\|_*+L\rho.
\end{equation}
Substituting the above into \cref{eq:sam_descent_mid_main} proves \cref{eq:sam_descent_final}.
\end{proof}

\begin{lemma}\label{lemma:sam_avg_stationarity}
Under \cref{ass:lower_bound,lipschiz,gradientvariance}, we have
\begin{equation}
\label{eq:sam_avg_grad_start_main}
\frac{1}{T}\sum_{t=1}^T \mathbb E\|\nabla f(W_{t-1})\|_*
\le
\frac{f(W_0)-f_\star}{\eta T}
+\frac{2}{T}\sum_{t=1}^T \mathbb E\|M_t-\nabla f(W_{t-1}+\delta_t)\|_*
+2L\rho+\frac{L\eta}{2},
\end{equation}
where $f_\star$ is a lower bound of $f$.
\end{lemma}

\begin{proof}
Taking expectation in \cref{eq:sam_descent_final}, summing over $t=1,\dots,T$, and using the telescoping identity
\[
\sum_{t=1}^T\bigl(f(W_t)-f(W_{t-1})\bigr)=f(W_T)-f(W_0),
\]
we obtain
\begin{align}
&\mathbb E[f(W_T)]-\mathbb E[f(W_0)]\notag\\
&\le
-\eta\sum_{t=1}^T \mathbb E\|\nabla f(W_{t-1})\|_*
+2\eta\sum_{t=1}^T \mathbb E\|M_t-\nabla f(W_{t-1}+\delta_t)\|_*
+2\eta LT\rho
+\frac{TL\eta^2}{2}.
\end{align}
Using $\mathbb E[f(W_T)]\ge f_\star$ and rearranging gives \cref{eq:sam_avg_grad_start_main}.
\end{proof}

\subsection{Theoretical Analysis for \cref{alg:musam}}

\begin{lemma}[Same-minibatch estimator decomposition]
\label{lemma:sam_mismatch_recursion}
According to \cref{lipschiz,gradientvariance}, we have
\begin{align}
&\mathbb E[\nabla\ell(W_{t-1};\mathcal B_t)
-
\nabla f(W_{t-1})
+
\nabla\ell(W_{t-1}+\delta_t;\mathcal B_t)-\nabla\ell(W_{t-1};\mathcal B_t)\notag\\
&-\mathbb E[\nabla\ell(W_{t-1}+\delta_t;\mathcal B_t)-\nabla\ell(W_{t-1};\mathcal B_t)\mid\mathcal F_{t-1}]\mid\mathcal F_{t-1}]=0
\label{eq:zeta_conditional_mean}
\end{align}
and
\begin{align}
&\mathbb E\![
\|\nabla\ell(W_{t-1};\mathcal B_t)
-
\nabla f(W_{t-1})
+
\nabla\ell(W_{t-1}+\delta_t;\mathcal B_t)-\nabla\ell(W_{t-1};\mathcal B_t)\notag\\
&-\mathbb E[\nabla\ell(W_{t-1}+\delta_t;\mathcal B_t)-\nabla\ell(W_{t-1};\mathcal B_t)\mid\mathcal F_{t-1}]\|_F^2
\mid
\mathcal F_{t-1}
]
\le
2\left(
\frac{\sigma^2}{B}
+
L^2\rho^2
\right).
\label{eq:zeta_conditional_variance1}
\end{align}

Define
\(
C_t:=M_t-\nabla f(W_{t-1}),
\)
for every \(t\ge2\),
\begin{align}
C_t
&=
\mu C_{t-1}
+
\widetilde D_t
+
(1-\mu)(
\nabla\ell(W_{t-1};\mathcal B_t)
-
\nabla f(W_{t-1})
+
\nabla\ell(W_{t-1}+\delta_t;\mathcal B_t)-\nabla\ell(W_{t-1};\mathcal B_t)\notag\\
&\quad-\mathbb E[\nabla\ell(W_{t-1}+\delta_t;\mathcal B_t)-\nabla\ell(W_{t-1};\mathcal B_t)\mid\mathcal F_{t-1}]),
\label{eq:Ct_recursion_main}
\end{align}
where
\begin{equation}
\widetilde D_t
:=
\mu\left(
\nabla f(W_{t-2})-\nabla f(W_{t-1})
\right)
+
(1-\mu)\mathbb E[\nabla\ell(W_{t-1}+\delta_t;\mathcal B_t)-\nabla\ell(W_{t-1};\mathcal B_t)\mid\mathcal F_{t-1}]
\label{eq:Dt_definition_main}
\end{equation}
is \(\mathcal F_{t-1}\)-measurable and satisfies
\begin{equation}
\|\widetilde D_t\|_F
\le
L\bigl(
\mu\eta+(1-\mu)\rho
\bigr).
\label{eq:Dt_bound_main}
\end{equation}
\end{lemma}

\begin{proof}
Because \(\nabla\ell(W_{t-1};\mathcal B_t)\) is evaluated at the
\(\mathcal F_{t-1}\)-measurable point \(W_{t-1}\),
\cref{gradientvariance} gives
\(
\mathbb E\!\left[
\nabla\ell(W_{t-1};\mathcal B_t)-\nabla f(W_{t-1})
\mid
\mathcal F_{t-1}
\right]
=0.
\)
Moreover, we have
\[
\mathbb E[\nabla\ell(W_{t-1}+\delta_t;\mathcal B_t)-\nabla\ell(W_{t-1};\mathcal B_t)-\mathbb E[\nabla\ell(W_{t-1}+\delta_t;\mathcal B_t)-\nabla\ell(W_{t-1};\mathcal B_t)\mid\mathcal F_{t-1}]\mid\mathcal F_{t-1}]=0.
\]
Adding these identities proves
\cref{eq:zeta_conditional_mean}.

By \cref{lipschiz} and \(\|\delta_t\|_F\le\rho\),
\begin{align}
\|\nabla\ell(W_{t-1}+\delta_t;\mathcal B_t)-\nabla\ell(W_{t-1};\mathcal B_t)\|_F
&\le
\|\nabla\ell(W_{t-1}+\delta_t;\mathcal B_t)-\nabla\ell(W_{t-1};\mathcal B_t)\|_*
\nonumber\\
&\le
L\|\delta_t\|_{\mathrm{op}}
\le
L\rho.
\label{eq:Rt_bound_main}
\end{align}
Consequently,
\begin{align}
&\|\mathbb E[\nabla\ell(W_{t-1}+\delta_t;\mathcal B_t)-\nabla\ell(W_{t-1};\mathcal B_t)\mid\mathcal F_{t-1}]\|_F\notag\\
&\le
\mathbb E[
\|\nabla\ell(W_{t-1}+\delta_t;\mathcal B_t)-\nabla\ell(W_{t-1};\mathcal B_t)\|_F
\mid
\mathcal F_{t-1}
]
\le
L\rho,
\end{align}
and the conditional variance identity gives
\begin{align}
&\mathbb E\!\left[
\|\nabla\ell(W_{t-1}\!+\!\delta_t;\mathcal B_t)\!-\!\nabla\ell(W_{t-1};\mathcal B_t)\!-\!\mathbb E[\nabla\ell(W_{t-1}\!+\!\delta_t;\mathcal B_t)\!-\!\nabla\ell(W_{t-1};\mathcal B_t)\!\mid\!\mathcal F_{t-1}]\|_F^2
\!\mid\!
\mathcal F_{t-1}
\!\right]\notag\\
&\le
\mathbb E\!\left[
\|\nabla\ell(W_{t-1}+\delta_t;\mathcal B_t)-\nabla\ell(W_{t-1};\mathcal B_t)\|_F^2
\mid
\mathcal F_{t-1}
\right]
\le
L^2\rho^2.
\end{align}
Using
\(\|A+B\|_F^2\le2\|A\|_F^2+2\|B\|_F^2\)
and \cref{gradientvariance}, we obtain
\begin{align}
&\mathbb E\![
\|\nabla\ell(W_{t-1};\mathcal B_t)
-
\nabla f(W_{t-1})
+
\nabla\ell(W_{t-1}+\delta_t;\mathcal B_t)-\nabla\ell(W_{t-1};\mathcal B_t)\notag\\
&\quad-\mathbb E[\nabla\ell(W_{t-1}+\delta_t;\mathcal B_t)-\nabla\ell(W_{t-1};\mathcal B_t)\mid\mathcal F_{t-1}]\|_F^2
\mid
\mathcal F_{t-1}
]\notag\\
&\le
2\,
\mathbb E\!\left[
\|\nabla\ell(W_{t-1};\mathcal B_t)-\nabla f(W_{t-1})\|_F^2
\mid
\mathcal F_{t-1}
\right]
\nonumber\\
&\quad+
\!2
\mathbb E[
\|\nabla\ell(W_{t-1}\!+\!\delta_t;\mathcal B_t)\!-\!\nabla\ell(W_{t-1};\mathcal B_t)\!-\!\mathbb E[\nabla\ell(W_{t-1}\!+\!\delta_t;\mathcal B_t)\!-\!\nabla\ell(W_{t-1};\mathcal B_t)\!\mid\!\mathcal F_{t-1}]\|_F^2
\!\!\mid\!
\mathcal F_{t-1}
]
\nonumber\\
&\le
2\left(
\frac{\sigma^2}{B}
+
L^2\rho^2
\right),
\end{align}
which proves \cref{eq:zeta_conditional_variance1}.

Next, the momentum recursion gives
\begin{align}
C_t
&=
\mu M_{t-1}
+
(1-\mu)\nabla\ell(W_{t-1}+\delta_t;\mathcal B_t)
-
\nabla f(W_{t-1})
\nonumber\\
&=
\mu\left(
M_{t-1}-\nabla f(W_{t-2})
\right)
+
\mu\left(
\nabla f(W_{t-2})-\nabla f(W_{t-1})
\right)
\nonumber\\
&\quad+
(1-\mu)\mathbb E[\nabla\ell(W_{t-1}+\delta_t;\mathcal B_t)-\nabla\ell(W_{t-1};\mathcal B_t)\mid\mathcal F_{t-1}]
\notag\\
&\quad+
(1-\mu)(\nabla\ell(W_{t-1};\mathcal B_t)
-
\nabla f(W_{t-1})
+
\nabla\ell(W_{t-1}+\delta_t;\mathcal B_t)-\nabla\ell(W_{t-1};\mathcal B_t)\notag\\
&\quad\quad\quad\quad\qquad-\mathbb E[\nabla\ell(W_{t-1}+\delta_t;\mathcal B_t)-\nabla\ell(W_{t-1};\mathcal B_t)\mid\mathcal F_{t-1}])
\nonumber\\
&=
\mu C_{t-1}
+
\widetilde D_t
+
(1-\mu)(
\nabla\ell(W_{t-1};\mathcal B_t)
-
\nabla f(W_{t-1})
+
\nabla\ell(W_{t-1}+\delta_t;\mathcal B_t)-\nabla\ell(W_{t-1};\mathcal B_t)\notag\\
&\quad-\mathbb E[\nabla\ell(W_{t-1}+\delta_t;\mathcal B_t)-\nabla\ell(W_{t-1};\mathcal B_t)\mid\mathcal F_{t-1}]).
\end{align}
This proves \cref{eq:Ct_recursion_main}.

Finally, \cref{lipschiz} and the exact-polar update imply
\[
\left\|
\nabla f(W_{t-2})-\nabla f(W_{t-1})
\right\|_F
\le
L\|W_{t-2}-W_{t-1}\|_{\mathrm{op}}
\le
L\eta.
\]
Combining this inequality with \(\|\mathbb E[\nabla\ell(W_{t-1}+\delta_t;\mathcal B_t)-\nabla\ell(W_{t-1};\mathcal B_t)\mid\mathcal F_{t-1}]\|_F\le L\rho\)
gives
\(
\|\widetilde D_t\|_F
\le
L\bigl(
\mu\eta+(1-\mu)\rho
\bigr),
\)
which proves \cref{eq:Dt_bound_main}.
\end{proof}

\begin{theorem}
\label{thm:sam_main_bound}
Suppose
\cref{ass:lower_bound,lipschiz,gradientvariance}
hold.

\paragraph{(i) Recursion for the clean-point estimator error.}
For every \(t\ge2\),
\begin{align}
\mathbb E\|C_t\|_F^2
&\le
\frac{1+\mu^2}{2}
\mathbb E\|C_{t-1}\|_F^2
+
\frac{1+\mu^2}{1-\mu^2}
L^2\bigl(
\mu\eta+(1-\mu)\rho
\bigr)^2
\nonumber\\
&\quad+
2(1-\mu)^2
\left(
\frac{\sigma^2}{B}
+
L^2\rho^2
\right).
\label{eq:At_recursion_main}
\end{align}

\paragraph{(ii) Unrolled bound.}
For every \(t\ge1\),
\begin{align}
\mathbb E\|C_t\|_F^2
&\le
\left(
\frac{1+\mu^2}{2}
\right)^{t-1}
\mathbb E\|C_1\|_F^2
+
\frac{2(1+\mu^2)}{(1-\mu^2)^2}
L^2\bigl(
\mu\eta+(1-\mu)\rho
\bigr)^2
\nonumber\\
&\quad+
\frac{4(1-\mu)}{1+\mu}
\left(
\frac{\sigma^2}{B}
+
L^2\rho^2
\right).
\label{eq:At_unrolled_main}
\end{align}

\paragraph{(iii) Pointwise mismatch bound.}
Let
\(
E_t
:=
M_t-\nabla f(W_{t-1}+\delta_t).
\)
Then, for every \(t\ge1\),
\begin{align}
\mathbb E\|E_t\|_F^2
&\le
2
\left(
\frac{1+\mu^2}{2}
\right)^{t-1}
\mathbb E\|C_1\|_F^2
+
\frac{4(1+\mu^2)}{(1-\mu^2)^2}
L^2\bigl(
\mu\eta+(1-\mu)\rho
\bigr)^2
\nonumber\\
&\quad+
\frac{8(1-\mu)}{1+\mu}
\left(
\frac{\sigma^2}{B}
+
L^2\rho^2
\right)
+
2L^2\rho^2.
\label{eq:Et_second_main}
\end{align}
Writing \(r:=\min\{m,n\}\), we further have
\begin{align}
\mathbb E\|E_t\|_*
&\le
\sqrt{2r}
\left(
\frac{1+\mu^2}{2}
\right)^{(t-1)/2}
\sqrt{\mathbb E\|C_1\|_F^2}+
\frac{2\sqrt{r(1+\mu^2)}}{1-\mu^2}
L\bigl(
\mu\eta+(1-\mu)\rho
\bigr)
\nonumber\\
&\quad+
2
\sqrt{
\frac{2r(1-\mu)}{1+\mu}
}
\sqrt{
\frac{\sigma^2}{B}+L^2\rho^2
}
+
\sqrt{2r}\,L\rho.
\label{eq:Et_first_main}
\end{align}

\paragraph{(iv) Averaged mismatch bound.}
For every \(T\ge1\),
\begin{align}
\frac{2}{T}
\sum_{t=1}^{T}
\mathbb E\|E_t\|_*
&\le
\frac{
8\sqrt{
2r\mathbb E\|C_1\|_F^2
}
}{
(1-\mu^2)T
}
+
\frac{
4\sqrt{r(1+\mu^2)}
}{
1-\mu^2
}
L\bigl(
\mu\eta+(1-\mu)\rho
\bigr)
\nonumber\\
&\quad+
4
\sqrt{
\frac{2r(1-\mu)}{1+\mu}
}
\sqrt{
\frac{\sigma^2}{B}+L^2\rho^2
}
+
2\sqrt{2r}\,L\rho.
\label{eq:avg_Et_main}
\end{align}

\paragraph{(v) Averaged stationarity bound.}
We have
\begin{align}
&\frac{1}{T}
\sum_{t=0}^{T-1}
\mathbb E\|\nabla f(W_t)\|_*
\nonumber\\
&\le
\frac{f(W_0)-f_\star}{\eta T}
+
\frac{8\sqrt{2r}}{(1-\mu^2)T}
\sqrt{
4\mu^2L\bigl(f(W_0)-f_\star\bigr)
+
4(1-\mu)^2L^2\rho^2
+
2(1-\mu)^2\frac{\sigma^2}{B}
}
\nonumber\\
&\quad+
\frac{4\sqrt{r(1+\mu^2)}}{1-\mu^2}
L\bigl(
\mu\eta+(1-\mu)\rho
\bigr)+
4
\sqrt{
\frac{2r(1-\mu)}{1+\mu}
}
\sqrt{
\frac{\sigma^2}{B}+L^2\rho^2
}
+
(2+2\sqrt{2r})L\rho
+
\frac{L\eta}{2}.
\label{eq:main_sam_bound_theorem}
\end{align}

In particular, if \(T\ge2\),
\(\eta=\rho=T^{-3/4}\), and \(1-\mu=T^{-1/2},
\)
then
\begin{align}
&\frac{1}{T}
\sum_{t=0}^{T-1}
\mathbb E\|\nabla f(W_t)\|_*\notag\\
&=
\mathcal O\!\left(
\frac{f(W_0)-f_\star+\sqrt r\,L}{T^{1/4}}
+
\frac{\sqrt r\,\sigma}{\sqrt B\,T^{1/4}}
+
\frac{
\sqrt{rL(f(W_0)-f_\star)}
}{
T^{1/2}
}
+
\frac{(1+\sqrt r)L}{T^{3/4}}
\right),
\label{eq:main_sam_rate_theorem}
\end{align}
and
\begin{align}
&\frac{1}{T}
\sum_{t=0}^{T-1}
\mathbb E
\|\nabla f(W_t+\delta_{t+1})\|_*\notag\\
&=
\mathcal O\!\left(
\frac{f(W_0)-f_\star+\sqrt r\,L}{T^{1/4}}
+
\frac{\sqrt r\,\sigma}{\sqrt B\,T^{1/4}}
+
\frac{
\sqrt{rL(f(W_0)-f_\star)}
}{
T^{1/2}
}
+
\frac{(1+\sqrt r)L}{T^{3/4}}
\right).
\label{eq:main_sam_rate_theorempur}
\end{align}
\end{theorem}

\begin{proof}
We prove the statements in order.

\paragraph{Proof of (i).}
By
\cref{eq:Ct_recursion_main,eq:zeta_conditional_mean},
conditioned on \(\mathcal F_{t-1}\), we have
\begin{align}
\mathbb E\!\left[
\|C_t\|_F^2
\mid
\mathcal F_{t-1}
\right]
&=
\left\|
\mu C_{t-1}+\widetilde D_t
\right\|_F^2
+
(1-\mu)^2
\mathbb E\!\left[
\|\zeta_t\|_F^2
\mid
\mathcal F_{t-1}
\right].
\label{eq:Ct_second_expand_main}
\end{align}
For any \(\alpha>0\), Young's inequality gives
\(
\|A+B\|_F^2
\le
(1+\alpha)\|A\|_F^2
+
\left(
1+\frac{1}{\alpha}
\right)
\|B\|_F^2.
\)
For \(\mu\in(0,1)\), apply this inequality with
\(A=\mu C_{t-1}\),
\(B=\widetilde D_t\), and
\(
\alpha
=
\frac{1-\mu^2}{2\mu^2}.
\)
Then
\(
\mu^2(1+\alpha)
=
\frac{1+\mu^2}{2}\) and \(1+\frac{1}{\alpha}
=
\frac{1+\mu^2}{1-\mu^2}.
\)
Using
\cref{eq:Dt_bound_main,eq:zeta_conditional_variance1}
in \cref{eq:Ct_second_expand_main}, and then taking total
expectation, gives
\begin{align}
\mathbb E\|C_t\|_F^2
\le
\frac{1+\mu^2}{2}
\mathbb E\|C_{t-1}\|_F^2
+
\frac{1+\mu^2}{1-\mu^2}
L^2\bigl(
\mu\eta+(1-\mu)\rho
\bigr)^2
+
2(1-\mu)^2
\left(
\frac{\sigma^2}{B}+L^2\rho^2
\right).
\end{align}
This proves \cref{eq:At_recursion_main}. For \(\mu=0\), \cref{eq:Ct_recursion_main} and conditional
orthogonality give
\(
\mathbb E\|C_t\|_F^2
\le
L^2\rho^2
+
2\left(\frac{\sigma^2}{B}+L^2\rho^2\right),
\)
which is bounded by the right-hand side of
\cref{eq:At_recursion_main} evaluated at \(\mu=0\). Hence, the same
recursion also holds for \(\mu=0\).

\paragraph{Proof of (ii).}
Let
\(
q:=\frac{1+\mu^2}{2}.
\)
Unrolling \cref{eq:At_recursion_main} from \(C_1\) gives
\begin{align}
\mathbb E\|C_t\|_F^2
\le
q^{t-1}\mathbb E\|C_1\|_F^2
+
\frac{1-q^{t-1}}{1-q}
\left[
\frac{1+\mu^2}{1-\mu^2}
L^2\bigl(
\mu\eta+(1-\mu)\rho
\bigr)^2
+
2(1-\mu)^2
\left(
\frac{\sigma^2}{B}+L^2\rho^2
\right)
\right].
\end{align}
Because
\(
1-q=\frac{1-\mu^2}{2},
\)
we obtain
\begin{align}
\mathbb E\|C_t\|_F^2
\le
q^{t-1}\mathbb E\|C_1\|_F^2
+
\frac{2(1+\mu^2)}{(1-\mu^2)^2}
L^2\bigl(
\mu\eta+(1-\mu)\rho
\bigr)^2+
\frac{4(1-\mu)}{1+\mu}
\left(
\frac{\sigma^2}{B}+L^2\rho^2
\right),
\end{align}
which proves \cref{eq:At_unrolled_main}.

\paragraph{Proof of (iii).}
By definition,
\(
E_t
=
C_t
-
\left(
\nabla f(W_{t-1}+\delta_t)
-
\nabla f(W_{t-1})
\right).
\)
Therefore,
\begin{align}
\|E_t\|_F^2
\le
2\|C_t\|_F^2
+
2
\left\|
\nabla f(W_{t-1}+\delta_t)
-
\nabla f(W_{t-1})
\right\|_F^2
\le
2\|C_t\|_F^2+2L^2\rho^2,
\end{align}
where we used
\cref{lipschiz} and \(\|\delta_t\|_F\le\rho\).
Combining this inequality with
\cref{eq:At_unrolled_main} proves
\cref{eq:Et_second_main}.

Furthermore,
\(
\|E_t\|_*
\le
\sqrt r\,\|E_t\|_F.
\)
Jensen's inequality and
\(\sqrt{x_1+\cdots+x_k}\le\sum_{i=1}^k\sqrt{x_i}\)
then give \cref{eq:Et_first_main}.

\paragraph{Proof of (iv).}
Summing \cref{eq:Et_first_main} over
\(t=1,\ldots,T\), and using
\begin{align}
\sum_{t=1}^{T}
\left(
\frac{1+\mu^2}{2}
\right)^{(t-1)/2}
\le
\frac{1}{
1-\sqrt{(1+\mu^2)/2}
}
=
\frac{
2\left(
1+\sqrt{(1+\mu^2)/2}
\right)
}{
1-\mu^2
}
\le
\frac{4}{1-\mu^2},
\end{align}
proves \cref{eq:avg_Et_main}.

\paragraph{Proof of (v).}
The definitions in
\cref{lemma:sam_mismatch_recursion} also apply at \(t=1\).
Since \(M_0=0\), we have
\begin{align}
C_1
&=
(1-\mu)\nabla\ell(W_{0}+\delta_1;\mathcal B_1)-\nabla f(W_0)\notag\\
&=
-\mu\nabla f(W_0)
+
(1-\mu)\mathbb E[\nabla\ell(W_{0}+\delta_1;\mathcal B_1)-\nabla\ell(W_{0};\mathcal B_1)\mid\mathcal F_{0}]\notag\\
&\quad+
(1-\mu)(\nabla\ell(W_{0};\mathcal B_1)
-
\nabla f(W_{0})
+
\nabla\ell(W_{0}+\delta_1;\mathcal B_1)-\nabla\ell(W_{0};\mathcal B_1)\notag\\
&\quad\quad\quad\quad\quad-\mathbb E[\nabla\ell(W_{0}+\delta_1;\mathcal B_1)-\nabla\ell(W_0;\mathcal B_1)\mid\mathcal F_0]).
\label{eq:C1_same_batch_decomposition}
\end{align}
Because
\(\mathbb E[\nabla\ell(W_{0};\mathcal B_1)
-
\nabla f(W_{0})
+
\nabla\ell(W_{0}+\delta_1;\mathcal B_1)-\nabla\ell(W_{0};\mathcal B_1)-\mathbb E[\nabla\ell(W_{0}+\delta_1;\mathcal B_1)-\nabla\ell(W_0;\mathcal B_1)\mid\mathcal F_0]\mid\mathcal F_0]=0\),
the cross term between
\(-\mu\nabla f(W_0)+(1-\mu)\mathbb E[\nabla\ell(W_{0}+\delta_1;\mathcal B_1)-\nabla\ell(W_{0};\mathcal B_1)\mid\mathcal F_{0}]\)
and \(\nabla\ell(W_{0};\mathcal B_1)
-
\nabla f(W_{0})
+
\nabla\ell(W_{0}+\delta_1;\mathcal B_1)-\nabla\ell(W_{0};\mathcal B_1)-\mathbb E[\nabla\ell(W_{0}+\delta_1;\mathcal B_1)-\nabla\ell(W_0;\mathcal B_1)\mid\mathcal F_0]\) vanishes. Hence,
\begin{align}
\mathbb E\|C_1\|_F^2
&=
\left\|
-\mu\nabla f(W_0)+(1-\mu)\mathbb E[\nabla\ell(W_{0}+\delta_1;\mathcal B_1)-\nabla\ell(W_{0};\mathcal B_1)\mid\mathcal F_{0}]
\right\|_F^2\notag\\
&\quad+
(1-\mu)^2
\mathbb E\|\nabla\ell(W_{0};\mathcal B_1)
-
\nabla f(W_{0})
+
\nabla\ell(W_{0}+\delta_1;\mathcal B_1)-\nabla\ell(W_{0};\mathcal B_1)\notag\\
&\quad\quad\quad\quad\quad\quad\quad-\mathbb E[\nabla\ell(W_{0}+\delta_1;\mathcal B_1)-\nabla\ell(W_0;\mathcal B_1)\mid\mathcal F_0]\|_F^2
\nonumber\\
&\le
2\mu^2\|\nabla f(W_0)\|_F^2
+
2(1-\mu)^2L^2\rho^2
+
2(1-\mu)^2
\left(
\frac{\sigma^2}{B}+L^2\rho^2
\right)
\nonumber\\
&=
2\mu^2\|\nabla f(W_0)\|_F^2
+
4(1-\mu)^2L^2\rho^2
+
2(1-\mu)^2\frac{\sigma^2}{B}.
\label{eq:A1_before_gap}
\end{align}
Under
\cref{ass:lower_bound,lipschiz}, the standard gradient-gap
inequality yields
\(
\|\nabla f(W_0)\|_F^2
\le
2L\bigl(f(W_0)-f_\star\bigr).
\)
Consequently,
\begin{equation}
\mathbb E\|C_1\|_F^2
\le
4\mu^2L\bigl(f(W_0)-f_\star\bigr)
+
4(1-\mu)^2L^2\rho^2
+
2(1-\mu)^2\frac{\sigma^2}{B}.
\label{eq:A1_gap_form}
\end{equation}

Substituting \cref{eq:A1_gap_form} into
\cref{eq:avg_Et_main}, and then applying
\cref{eq:sam_avg_grad_start_main}, gives
\cref{eq:main_sam_bound_theorem}.

To recover the simplified main-text bound, use
\(
\sqrt{
\frac{\sigma^2}{B}+L^2\rho^2
}
\le
\frac{\sigma}{\sqrt B}+L\rho
\)
and
\(
\sqrt{\frac{1-\mu}{1+\mu}}\le1.
\)
This yields
\begin{align}
\frac{1}{T}
\sum_{t=0}^{T-1}
\mathbb E\|\nabla f(W_t)\|_*
&\lesssim
\frac{f(W_0)-f_\star}{\eta T}
+
\frac{\sqrt{r(1+\mu^2)}}{1-\mu^2}
L\bigl(
\mu\eta+(1-\mu)\rho
\bigr)
\nonumber\\
&\quad+
\frac{\sqrt r}{(1-\mu^2)T}
\sqrt{
\mu^2L\bigl(f(W_0)-f_\star\bigr)
+
(1-\mu)^2L^2\rho^2
+
(1-\mu)^2\frac{\sigma^2}{B}
}
\nonumber\\
&\quad+
\sqrt{
\frac{r(1-\mu)}{1+\mu}
}
\left(
\frac{\sigma}{\sqrt B}+L\rho
\right)
+
(1+\sqrt r)L\rho
+
L\eta.
\label{eq:main_sam_bound_gap}
\end{align}

Now set
\(
\eta=\rho=T^{-3/4}\) and
\(1-\mu=T^{-1/2},
\)
then 
\cref{eq:main_sam_bound_gap} implies
\begin{align}
&\frac{1}{T}
\sum_{t=0}^{T-1}
\mathbb E\|\nabla f(W_t)\|_*\notag\\
&=
\mathcal O\!\left(
\frac{f(W_0)-f_\star+\sqrt r\,L}{T^{1/4}}
+
\frac{\sqrt r\,\sigma}{\sqrt B\,T^{1/4}}
+
\frac{
\sqrt{rL(f(W_0)-f_\star)}
}{
T^{1/2}
}
+
\frac{(1+\sqrt r)L}{T^{3/4}}
\right).
\label{eq:main_sam_rate_gap}
\end{align}

Finally, for every \(t\),
\begin{align}
\|\nabla f(W_t+\delta_{t+1})\|_*
&\le
\|\nabla f(W_t)\|_*
+
\left\|
\nabla f(W_t+\delta_{t+1})
-
\nabla f(W_t)
\right\|_*
\nonumber\\
&\le
\|\nabla f(W_t)\|_*
+
L\|\delta_{t+1}\|_{\mathrm{op}}
\nonumber\\
&\le
\|\nabla f(W_t)\|_*+L\rho.
\label{eqepsilon}
\end{align}
Taking expectation and averaging gives
\begin{equation}
\frac{1}{T}
\sum_{t=0}^{T-1}
\mathbb E
\|\nabla f(W_t+\delta_{t+1})\|_*
\le
\frac{1}{T}
\sum_{t=0}^{T-1}
\mathbb E\|\nabla f(W_t)\|_*
+
L\rho.
\label{eq:perturbed_grad_avg_bound}
\end{equation}
By combining \cref{eq:perturbed_grad_avg_bound} and \cref{eq:main_sam_bound_gap}, and further setting $\eta=\rho=T^{-3/4}$ and $1-\mu=T^{-1/2}$, we obtain
\begin{align}
&\frac{1}{T}
\sum_{t=0}^{T-1}
\mathbb E
\|\nabla f(W_t+\delta_{t+1})\|_*\notag\\
&\lesssim
\frac{f(W_0)-f_\star}{\eta T}
+
\frac{\sqrt{r(1+\mu^2)}}{1-\mu^2}
L\bigl(
\mu\eta+(1-\mu)\rho
\bigr)
\nonumber\\
&\quad+
\frac{\sqrt r}{(1-\mu^2)T}
\sqrt{
\mu^2L\bigl(f(W_0)-f_\star\bigr)
+
(1-\mu)^2L^2\rho^2
+
(1-\mu)^2\frac{\sigma^2}{B}
}
\nonumber\\
&\quad+
\sqrt{
\frac{r(1-\mu)}{1+\mu}
}
\left(
\frac{\sigma}{\sqrt B}+L\rho
\right)
+
(1+\sqrt r)L\rho
+
L\eta\notag\\
&=\mathcal O\!\left(
\frac{f(W_0)-f_\star+\sqrt r\,L}{T^{1/4}}
+
\frac{\sqrt r\,\sigma}{\sqrt B\,T^{1/4}}
+
\frac{
\sqrt{rL(f(W_0)-f_\star)}
}{
T^{1/2}
}
+
\frac{(1+\sqrt r)L}{T^{3/4}}
\right).
\label{eq:main_sam_bound_gap111}
\end{align}

\end{proof}

\newpage
\subsection{Theoretical Analysis for \cref{alg:momusam}}\label{subsectionb3}
\begin{algorithm}[t]
\caption{Momentum-guided Sharpness-Aware Muon (\textsc{MoMuSAM})}
\label{alg:momusam}
\begin{algorithmic}[1]
\REQUIRE Loss \(\ell(W)\), initial matrix parameter \(W_0 \in \mathbb{R}^{m\times n}\), learning rate \(\eta\), momentum coefficient \(\mu\in[0,1)\), perturbation radius \(\rho>0\)
\STATE Initialize \(M_0 \gets 0\) and \(\delta_1 \gets 0\)
\FOR{\(t = 1,2,\dots\)}
    \STATE Compute perturbed stochastic gradient
    $\nabla \ell(W_{t-1}+\delta_t;\mathcal B_t)$
    \STATE Update momentum matrix
    $M_t \gets \mu M_{t-1} + (1-\mu)\nabla \ell(W_{t-1}+\delta_t;\mathcal B_t)$
    \STATE Orthogonalize the momentum matrix
    $O_t \in \arg\min_{O\in\mathbb{R}^{m\times n}} \|O-M_t\|_F,
    \!\!\!\!\!\quad
    \text{s.t.}\quad\!\!\!\!
    O^\top O = I_n$
    \STATE Update parameter
    $W_t \gets W_{t-1} - \eta O_t$
     \STATE Construct momentum-guided perturbation $\delta_{t+1} \gets \rho \frac{M_t}{\|M_t\|_F}$
\ENDFOR
\end{algorithmic}
\end{algorithm}

Let $E_t=M_t-\nabla f(W_{t-1}+\delta_t)$.
Using $M_t=\mu M_{t-1}+(1-\mu)\nabla \ell(W_{t-1}+\delta_t)$,
we have
\begin{align}
E_t
&=
\mu M_{t-1} + (1-\mu)\nabla \ell(W_{t-1}+\delta_t) -\nabla f(W_{t-1}+\delta_t)
\nonumber\\
&=
\mu(M_{t-1}-\nabla f(W_{t-2}+\delta_{t-1}))
+\mu(\nabla f(W_{t-2}+\delta_{t-1})-\nabla f(W_{t-1}+\delta_t))\notag\\
&\quad+(1-\mu)(\nabla \ell(W_{t-1}+\delta_t)-\nabla f(W_{t-1}+\delta_t))
\nonumber\\
&=
\mu E_{t-1}
+\mu(\nabla f(W_{t-2}+\delta_{t-1})-\nabla f(W_{t-1}+\delta_t))\notag\\
&\quad+(1-\mu)(\nabla \ell(W_{t-1}+\delta_t)-\nabla f(W_{t-1}+\delta_t)).
\label{eq:E_t_recursion_re}
\end{align}
Hence
\begin{align}
\label{eq:E_t_norm_re}
\|E_t\|_*
&\le
\mu\|E_{t-1}\|_*
+\mu\|\nabla f(W_{t-2}+\delta_{t-1})-\nabla f(W_{t-1}+\delta_t)\|_*\notag\\
&\quad+(1-\mu)\|\nabla \ell(W_{t-1}+\delta_t)-\nabla f(W_{t-1}+\delta_t)\|_*.
\end{align}

Conditioned on $\mathcal F_{t-1}$, both $E_{t-1}$ and $\nabla f(W_{t-2}+\delta_{t-1})-\nabla f(W_{t-1}+\delta_t)$ are deterministic. Therefore, the cross term vanishes:
\begin{align}
\mathbb E\!\left[\|E_t\|_F^2 \mid \mathcal F_{t-1}\right]
&=
\mathbb E\!\left[\|\mu(E_{t-1}+\nabla f(W_{t-2}+\delta_{t-1})-\nabla f(W_{t-1}+\delta_t))\right.\notag\\
&\quad\quad\quad\left.+(1-\mu)(\nabla \ell(W_{t-1}+\delta_t)-\nabla f(W_{t-1}+\delta_t))\|_F^2 \mid \mathcal F_{t-1}\right]
\nonumber\\
&=
\mu^2\|E_{t-1}+\nabla f(W_{t-2}+\delta_{t-1})-\nabla f(W_{t-1}+\delta_t)\|_F^2\notag\\
&\quad+
(1-\mu)^2\mathbb E\!\left[\|(\nabla \ell(W_{t-1}+\delta_t)-\nabla f(W_{t-1}+\delta_t))\|_F^2 \mid \mathcal F_{t-1}\right].
\label{eq:Et_second_expand}
\end{align}

By \cref{lipschiz}, we have
\begin{align}
\|\nabla f(W_{t-2}+\delta_{t-1})-\nabla f(W_{t-1}+\delta_t)\|_F
&\le 
\|\nabla f(W_{t-2}+\delta_{t-1})-\nabla f(W_{t-1}+\delta_t)\|_*
\nonumber\\
&\le
L\|(W_{t-2}+\delta_{t-1})-(W_{t-1}+\delta_t)\|_{\mathrm{op}}
\nonumber\\
&\le
L\|W_{t-1}-W_{t-2}\|_{\mathrm{op}}
+
L\|\delta_t-\delta_{t-1}\|_{\mathrm{op}}.
\end{align}
Since $W_{t-1}=W_{t-2}-\eta U_{t-1}$ and $\|U_{t-1}\|_{\mathrm{op}}\le 1$,
\[
\|W_{t-1}-W_{t-2}\|_{\mathrm{op}}\le \eta.
\]
Also $\|\delta_t\|_{\mathrm{op}}\le \|\delta_t\|_F=\rho$, hence
\[
\|\delta_t-\delta_{t-1}\|_{\mathrm{op}}\le 2\rho.
\]
Therefore,
\begin{equation}
\label{eq:Dt_bound}
\|\nabla f(W_{t-2}+\delta_{t-1})-\nabla f(W_{t-1}+\delta_t)\|_F \le L(\eta+2\rho).
\end{equation}

For any $\alpha>0$, we have
\[
\|A+B\|_F^2 \le (1+\alpha)\|A\|_F^2 + \Bigl(1+\frac{1}{\alpha}\Bigr)\|B\|_F^2.
\]
Applying the inequality above together with \cref{gradientvariance} to $A=E_{t-1}$,  $B=\nabla f(W_{t-2}+\delta_{t-1})-\nabla f(W_{t-1}+\delta_t)$ in \cref{eq:Et_second_expand}, we obtain
\begin{align}
&\mathbb E\!\left[\|E_t\|_F^2 \mid \mathcal F_{t-1}\right]\notag\\
&\le
\mu^2(1+\alpha)\|E_{t-1}\|_F^2
+
\mu^2\Bigl(1+\frac{1}{\alpha}\Bigr)\|\nabla f(W_{t-2}+\delta_{t-1})-\nabla f(W_{t-1}+\delta_t)\|_F^2
+
(1-\mu)^2\frac{\sigma^2}{B}.
\label{eq:Et_second_young}
\end{align}
For \(\mu\in(0,1)\), choose
\(
\alpha := \frac{1-\mu^2}{2\mu^2}.
\)
Then $\mu^2(1+\alpha)=\frac{1+\mu^2}{2}$
and
$\mu^2\Bigl(1+\frac{1}{\alpha}\Bigr)=\mu^2\frac{1+\mu^2}{1-\mu^2}.$
Using \cref{eq:Dt_bound}, we get
\begin{equation}
\label{eq:Et_second_recursion}
\mathbb E\!\left[\|E_t\|_F^2 \mid \mathcal F_{t-1}\right]
\le
\frac{1+\mu^2}{2}\|E_{t-1}\|_F^2
+
\frac{\mu^2(1+\mu^2)}{1-\mu^2}L^2(\eta+2\rho)^2
+
(1-\mu)^2\frac{\sigma^2}{B}.
\end{equation}

For \(\mu=0\), we have
\(
E_t
=
\nabla\ell(W_{t-1}+\delta_t;\mathcal B_t)
-
\nabla f(W_{t-1}+\delta_t),
\)
and hence
\(
\mathbb E[\|E_t\|_F^2\mid\mathcal F_{t-1}]
\le\frac{\sigma^2}{B}
\)
by \cref{gradientvariance}. This is bounded by the right-hand side
of \cref{eq:Et_second_recursion} evaluated at \(\mu=0\), so the same
recursion remains valid.

Then \cref{eq:Et_second_recursion} implies
\begin{equation}
\label{eq:at_recursion}
\mathbb E\|E_t\|_F^2 \le \frac{1+\mu^2}{2}\mathbb E\|E_{t-1}\|_F^2+
\frac{\mu^2(1+\mu^2)}{1-\mu^2}L^2(\eta+2\rho)^2+(1-\mu)^2\frac{\sigma^2}{B}.
\end{equation}
Unrolling gives
\begin{align}
\label{eq:at_unrolled}
\mathbb E\|E_t\|_F^2 &\le
 \left(\frac{1+\mu^2}{2}\right)^{t-1}\mathbb E\|E_{1}\|_F^2 + \frac{1-\left(\frac{1+\mu^2}{2}\right)^{t-1}}{1-\left(\frac{1+\mu^2}{2}\right)}\left(
\frac{\mu^2(1+\mu^2)}{1-\mu^2}L^2(\eta+2\rho)^2+(1-\mu)^2\frac{\sigma^2}{B}\right)\notag\\
&\le \left(\frac{1+\mu^2}{2}\right)^{t-1}\mathbb E\|E_{1}\|_F^2 + \frac{2}{1-\mu^2}\left(
\frac{\mu^2(1+\mu^2)}{1-\mu^2}L^2(\eta+2\rho)^2+(1-\mu)^2\frac{\sigma^2}{B}\right)\notag\\
&=\left(\frac{1+\mu^2}{2}\right)^{t-1}\mathbb E\|E_{1}\|_F^2 +
\frac{2\mu^2(1+\mu^2)}{(1-\mu^2)^2}L^2(\eta+2\rho)^2
+
\frac{2(1-\mu)}{1+\mu}\frac{\sigma^2}{B}.
\end{align}

Let $r=\min\{m,n\}$. Since $\|A\|_* \le \sqrt r\,\|A\|_F$, Jensen's inequality gives
\begin{align}
\mathbb E\|E_t\|_*
&\le
\sqrt r\,\mathbb E\|E_t\|_F
\le
\sqrt r\,\sqrt{\mathbb E\|E_t\|_F^2}.
\end{align}
Using \cref{eq:at_unrolled} and $\sqrt{x+y+z}\le \sqrt x+\sqrt y+\sqrt z$, we obtain
\begin{align}
\mathbb E\|E_t\|_*
&\le
\sqrt r\,\left(\frac{1+\mu^2}{2}\right)^{\frac{t-1}{2}}\sqrt{\mathbb E\|E_{1}\|_F^2}
+
\frac{\sqrt{2r}\,\mu\sqrt{1+\mu^2}}{1-\mu^2}\,L(\eta+2\rho)
+
\sqrt{\frac{2r(1-\mu)}{1+\mu}}\;\frac{\sigma}{\sqrt B}.
\label{eq:Et_first_moment_refined}
\end{align}

Summing \cref{eq:Et_first_moment_refined} over $t=1,\dots,T$ gives
\begin{align}
&\sum_{t=1}^T \mathbb E\|E_t\|_*\notag\\
&\le
\sqrt r\,\sqrt{\mathbb E\|E_{1}\|_F^2}\sum_{t=1}^T \left(\frac{1+\mu^2}{2}\right)^{\frac{t-1}{2}}
+
\frac{\sqrt{2r}\,\mu\sqrt{1+\mu^2}}{1-\mu^2}TL(\eta+2\rho)
+
\sqrt{\frac{2r(1-\mu)}{1+\mu}}\;\frac{T\sigma}{\sqrt B}.
\end{align}
Since
\(
\sum_{t=1}^T \left(\frac{1+\mu^2}{2}\right)^{\frac{t-1}{2}}
=
\sum_{s=0}^{T-1}\left(\sqrt{\frac{1+\mu^2}{2}}\right)^s
=
\frac{1-\left(\sqrt{\frac{1+\mu^2}{2}}\right)^T}{1-\sqrt{\frac{1+\mu^2}{2}}},
\)
we obtain
\begin{align}
\label{eq:avg_Et_refined}
\frac{2}{T}\sum_{t=1}^T \mathbb E\|E_t\|_*
&\le
\frac{2\sqrt r\,\sqrt{\mathbb E\|E_{1}\|_F^2}}{T}\cdot
\frac{1-\left(\sqrt{\frac{1+\mu^2}{2}}\right)^T}{1-\sqrt{\frac{1+\mu^2}{2}}}
+
\frac{2\sqrt{2r}\,\mu\sqrt{1+\mu^2}}{1-\mu^2}\,L(\eta+2\rho)
\notag\\
&\quad+
2\sqrt{\frac{2r(1-\mu)}{1+\mu}}\;\frac{\sigma}{\sqrt B}.
\end{align}
Using
\(
\frac{1-\left(\sqrt{\frac{1+\mu^2}{2}}\right)^T}{1-\sqrt{\frac{1+\mu^2}{2}}}
\le
\frac{4}{1-\mu^2},
\)
we further have
\begin{align}
\label{eq:avg_Et_refined_simple}
\frac{2}{T}\sum_{t=1}^T \mathbb E\|E_t\|_*
&\le
\frac{8\sqrt r\,\sqrt{\mathbb E\|E_{1}\|_F^2}}{(1-\mu^2)T}
+
\frac{2\sqrt{2r}\,\mu\sqrt{1+\mu^2}}{1-\mu^2}\,L(\eta+2\rho)
\notag\\
&\quad+
2\sqrt{\frac{2r(1-\mu)}{1+\mu}}\;\frac{\sigma}{\sqrt B}.
\end{align}

By combining equation \cref{eq:sam_avg_grad_start_main} and \cref{eq:avg_Et_refined}, we have
\begin{align}
\label{eq:avg_truegradfinal}
&\frac{1}{T}\sum_{t=0}^{T-1} \mathbb E\|\nabla f(W_{t})\|_*\notag\\
&\le
\frac{f(W_0)-f_*}{\eta T}
+\frac{2}{T}\sum_{t=1}^T \mathbb{E}\|E_t\|_*
+2L\rho
+\frac{L\eta}{2}\notag\\
&\le
\frac{f(W_0)-f_*}{\eta T}
+\frac{8\sqrt r}{(1-\mu^2)T}\sqrt{
2\mu^2L(f(W_0)-f_*)
+
(1-\mu)^2\frac{\sigma^2}{B}}\notag\\
&\quad+
\frac{2\sqrt{2r}\,\mu\sqrt{1+\mu^2}}{1-\mu^2}\,L(\eta+2\rho)+
2\sqrt{\frac{2r(1-\mu)}{1+\mu}}\;\frac{\sigma}{\sqrt B}
+2L\rho
+\frac{L\eta}{2}\notag\\
&\lesssim
\frac{f(W_0)-f_*}{\eta T}
+\frac{\sqrt r}{(1-\mu^2)T}\sqrt{
\mu^2L(f(W_0)-f_*)
+
(1-\mu)^2\frac{\sigma^2}{B}}\notag\\
&\quad+
\frac{\sqrt{r}\,\mu\sqrt{1+\mu^2}}{1-\mu^2}\,L(\eta+\rho)+
\sqrt{\frac{r(1-\mu)}{1+\mu}}\;\frac{\sigma}{\sqrt B}
+L\rho
+{L\eta}.
\end{align}

Let $\eta=\rho=\frac{1}{T^\frac{3}{4}}$ and $1-\mu=\frac{1}{T^\frac{1}{2}}$, we obtain
\begin{align}
\label{eq:avg_truegradfinal1}
&\frac{1}{T}\sum_{t=0}^{T-1} \mathbb E\|\nabla f(W_{t})\|_*\notag\\
&\le\frac{f(W_0)-f_*+\sqrt r L}{T^\frac{1}{4}}+
\frac{\sqrt{{r}}\sigma}{\sqrt B T^\frac{1}{4}}
+\frac{\sqrt{rL({f(W_0)-f_*})}}{T^\frac{1}{2}}
+\frac{L}{T^\frac{3}{4}}\notag\\
&=
\mathcal O\!\left(
\frac{f(W_0)-f_\star+\sqrt r\,L}{T^{1/4}}
+
\frac{\sqrt r\,\sigma}{\sqrt B\,T^{1/4}}
+
\frac{\sqrt{rL(f(W_0)-f_\star)}}{T^{1/2}}
+
\frac{L}{T^{3/4}}
\right).
\end{align}

Moreover, according to \cref{eqepsilon}, if we set $\eta=\rho=\frac{1}{T^\frac{3}{4}}$ and $1-\mu=\frac{1}{T^\frac{1}{2}}$, we further have
\begin{align}
\label{eq:main_perturbed_grad_rate_latex}
&\frac{1}{T}\sum_{t=0}^{T-1} \mathbb E\|\nabla f(W_{t}+\delta_{t+1})\|_*\notag\\
&=
\mathcal O\!\left(
\frac{f(W_0)-f_\star+\sqrt r\,L}{T^{1/4}}
+
\frac{\sqrt r\,\sigma}{\sqrt B\,T^{1/4}}
+
\frac{\sqrt{rL(f(W_0)-f_\star)}}{T^{1/2}}
+
\frac{L}{T^{3/4}}
\right).
\end{align}

\newpage

\subsection{Proof of \cref{thm:linearization_stability}}
\label{app:proof_linearization_stability}

We first establish a direction-sensitive comparison between a slope proxy
and the population-gradient perturbation.

\begin{lemma}[Achieved-linearization deviation]
\label{lem:achieved_linearization_deviation}
For every \(V\in\mathbb R^{m\times n}\),
\begin{equation}
0
\le
\mathcal E_t(V)
\le
2\rho
\left\|
V-\nabla f(W_{t-1})
\right\|_F.
\label{eq:achieved_linearization_deviation}
\end{equation}
\end{lemma}

\begin{proof}
Let \(g_t:=\nabla f(W_{t-1})\). By
\cref{eq:directional_linearization_gap},
\begin{equation}
\mathcal E_t(V)
=
\rho
\left\langle
g_t,\mathsf n(g_t)-\mathsf n(V)
\right\rangle.
\label{eq:directional_gap_expansion}
\end{equation}
Since \(\mathsf n(g_t)\) maximizes
\(\langle g_t,U\rangle\) over \(\|U\|_F\le1\), it follows immediately
that \(\mathcal E_t(V)\ge0\).

Furthermore,
\begin{align}
\mathcal E_t(V)
=
\rho
\left\langle
g_t-V,\mathsf n(g_t)-\mathsf n(V)
\right\rangle+
\rho
\left\langle
V,\mathsf n(g_t)-\mathsf n(V)
\right\rangle.
\label{eq:directional_gap_decomposition}
\end{align}
Because \(\mathsf n(V)\) maximizes
\(\langle V,U\rangle\) over \(\|U\|_F\le1\),
\begin{equation}
\left\langle
V,\mathsf n(g_t)-\mathsf n(V)
\right\rangle
\le0.
\label{eq:proxy_optimality}
\end{equation}
Therefore, by Cauchy--Schwarz,
\begin{align}
\mathcal E_t(V)
&\le
\rho
\left\|
g_t-V
\right\|_F
\left\|
\mathsf n(g_t)-\mathsf n(V)
\right\|_F
\nonumber\\
&\le
2\rho
\left\|
g_t-V
\right\|_F,
\end{align}
where the last inequality uses
\(
\|\mathsf n(g_t)\|_F\le1
\)
and
\(
\|\mathsf n(V)\|_F\le1.
\)
The argument also covers \(g_t=0\) or \(V=0\) under the convention
\(\mathsf n(0)=0\).
\end{proof}

\subsubsection{Stability of the \textsc{MuSAM} perturbation}

Let $G_t^{\mathrm{clean}}
:=
\nabla\ell(W_{t-1};\mathcal B_t)$
denote the raw stochastic gradient used to construct the
\textsc{MuSAM} perturbation. Applying
\cref{lem:achieved_linearization_deviation} gives
\begin{equation}
\mathcal E_t(G_t^{\mathrm{clean}})
\le
2\rho
\left\|
G_t^{\mathrm{clean}}-\nabla f(W_{t-1})
\right\|_F.
\label{eq:musam_directional_first}
\end{equation}
Taking expectation and applying Jensen's inequality together with
\cref{gradientvariance}, we obtain
\begin{align}
\mathbb E\!\left[
\mathcal E_t(G_t^{\mathrm{clean}})
\right]
&\le
2\rho\,
\mathbb E\!\left[
\left\|
G_t^{\mathrm{clean}}-\nabla f(W_{t-1})
\right\|_F
\right]
\nonumber\\
&\le
2\rho
\sqrt{
\mathbb E\!\left[
\left\|
G_t^{\mathrm{clean}}-\nabla f(W_{t-1})
\right\|_F^2
\right]
}
\nonumber\\
&\le
\frac{2\rho\sigma}{\sqrt B}.
\label{eq:musam_stability_proof}
\end{align}
This proves \cref{eq:musam_stability}.

\subsubsection{Stability of the \textsc{MoMuSAM} perturbation}

For \(s\ge1\), define $X_s:=W_{s-1}+\delta_s$ and $G_s:=\nabla\ell(X_s;\mathcal B_s)$.
Unrolling the momentum recursion from \(M_0=0\) gives, for every
\(t\ge2\),
\begin{equation}
M_{t-1}
=
(1-\mu)
\sum_{s=1}^{t-1}
\mu^{t-1-s}G_s.
\label{eq:proof_momentum_unrolled}
\end{equation}

Define the bias-corrected momentum
$\overline M_{t-1}
:=
\frac{M_{t-1}}{1-\mu^{t-1}}
=
\sum_{s=1}^{t-1}w_{t,s}G_s$,
where
$
w_{t,s}
:=
\frac{
(1-\mu)\mu^{t-1-s}
}{
1-\mu^{t-1}
}.$
These weights are nonnegative and satisfy
$\sum_{s=1}^{t-1}w_{t,s}=1.$

Since \(1-\mu^{t-1}>0\), positive rescaling does not change the
normalized direction: $\mathsf n(\overline M_{t-1})
=
\mathsf n(M_{t-1})$.
Consequently, $\mathcal E_t(\overline M_{t-1})
=
\mathcal E_t(M_{t-1})$.
Thus, \(\overline M_{t-1}\) is used only for the analysis and induces
exactly the same perturbation direction as the momentum appearing in
\textsc{MoMuSAM}.

Define the stochastic-gradient noise $Z_s
:=
G_s-\nabla f(X_s)$.
We decompose the bias-corrected momentum error as
\begin{align}
\overline M_{t-1}-\nabla f(W_{t-1})
=
\underbrace{
\sum_{s=1}^{t-1}
w_{t,s}
\left(
\nabla f(X_s)-\nabla f(W_{t-1})
\right)
}_{\mathsf B_t}+
\underbrace{
\sum_{s=1}^{t-1}w_{t,s}Z_s
}_{\mathsf N_t}.
\label{eq:momentum_error_decomposition}
\end{align}
Here, \(\mathsf B_t\) is the tracking component and
\(\mathsf N_t\) is the martingale-noise component.

\paragraph{Bounding the tracking component.}
The exact-polar update satisfies
$W_j=W_{j-1}-\eta O_j$ and $\|O_j\|_{\mathrm{op}}\le1$.
Hence, for every \(s\le t-1\),
\begin{equation}
\|W_{t-1}-W_{s-1}\|_{\mathrm{op}}
\le
\sum_{j=s}^{t-1}
\|W_j-W_{j-1}\|_{\mathrm{op}}
\le
(t-s)\eta.
\label{eq:iterate_distance_stability}
\end{equation}
Since
\(
\|\delta_s\|_{\mathrm{op}}
\le
\|\delta_s\|_F
\le\rho,
\)
we have
\begin{equation}
\|X_s-W_{t-1}\|_{\mathrm{op}}
\le
(t-s)\eta+\rho.
\label{eq:query_current_distance}
\end{equation}

Using
\(
\|A\|_F\le\|A\|_*
\),
\cref{lipschiz}, and
$\sum_{s=1}^{t-1}w_{t,s}=1$, we obtain
\begin{align}
\|\mathsf B_t\|_F
&\le
\sum_{s=1}^{t-1}
w_{t,s}
\left\|
\nabla f(X_s)-\nabla f(W_{t-1})
\right\|_F
\nonumber\\
&\le
\sum_{s=1}^{t-1}
w_{t,s}
\left\|
\nabla f(X_s)-\nabla f(W_{t-1})
\right\|_*
\nonumber\\
&\le
L
\sum_{s=1}^{t-1}
w_{t,s}
\|X_s-W_{t-1}\|_{\mathrm{op}}
\nonumber\\
&\le
L\left(
\rho+
\eta
\sum_{s=1}^{t-1}w_{t,s}(t-s)
\right).
\label{eq:tracking_component_first}
\end{align}

The weighted delay admits the exact expression
\begin{align}
A_t(\mu)
&:=
\sum_{s=1}^{t-1}w_{t,s}(t-s)
\nonumber\\
&=
\frac{1-\mu}{1-\mu^{t-1}}
\sum_{k=0}^{t-2}(k+1)\mu^k
\nonumber\\
&=
\frac{
1-t\mu^{t-1}+(t-1)\mu^t
}{
(1-\mu)(1-\mu^{t-1})
}.
\label{eq:At_calculation}
\end{align}
Moreover, $1-t\mu^{t-1}+(t-1)\mu^t
\le
1-\mu^{t-1}$,
because the difference between the right- and left-hand sides is
$
(t-1)\mu^{t-1}(1-\mu)\ge0.$
Therefore,
$A_t(\mu)
\le
\frac{1}{1-\mu}$.
Substituting this into
\cref{eq:tracking_component_first} yields
\begin{equation}
\|\mathsf B_t\|_F
\le
L\bigl(\rho+\eta A_t(\mu)\bigr)
\le
L\left(
\rho+\frac{\eta}{1-\mu}
\right).
\label{eq:tracking_component_bound}
\end{equation}

\paragraph{Bounding the martingale-noise component.}
Let \(\mathcal F_{s-1}\) denote the history before sampling
\(\mathcal B_s\). Since \(X_s\) is
\(\mathcal F_{s-1}\)-measurable, the stochastic-oracle assumption gives
$\mathbb E[Z_s\mid\mathcal F_{s-1}]=0$ and 
$\mathbb E[
\|Z_s\|_F^2
\mid
\mathcal F_{s-1}
]
\le
\frac{\sigma^2}{B}$.
For \(i<j\), \(Z_i\) is \(\mathcal F_{j-1}\)-measurable. Thus,
\begin{align}
\mathbb E\langle Z_i,Z_j\rangle
=
\mathbb E\!\left[
\mathbb E[
\langle Z_i,Z_j\rangle
\mid
\mathcal F_{j-1}
]
\right]
=
\mathbb E\!\left[
\left\langle
Z_i,
\mathbb E[Z_j\mid\mathcal F_{j-1}]
\right\rangle
\right]
=0.
\label{eq:noise_cross_terms}
\end{align}
It follows that
\begin{align}
\mathbb E\|\mathsf N_t\|_F^2
&=
\sum_{s=1}^{t-1}
w_{t,s}^2
\mathbb E\|Z_s\|_F^2
\nonumber\\
&\le
\frac{\sigma^2}{B}
\sum_{s=1}^{t-1}w_{t,s}^2
\nonumber\\
&=
\frac{\sigma^2}{B}
\frac{(1-\mu)^2}{(1-\mu^{t-1})^2}
\sum_{k=0}^{t-2}\mu^{2k}
\nonumber\\
&=
\frac{\sigma^2}{B}
\frac{1-\mu}{1+\mu}
\frac{1+\mu^{t-1}}{1-\mu^{t-1}}.
\label{eq:noise_component_bound}
\end{align}
Define $\Gamma_t(\mu)
:=
\frac{1-\mu}{1+\mu}
\frac{1+\mu^{t-1}}{1-\mu^{t-1}}.$
By Jensen's inequality,
\begin{equation}
\mathbb E\|\mathsf N_t\|_F
\le
\sqrt{\Gamma_t(\mu)}
\frac{\sigma}{\sqrt B}.
\label{eq:noise_component_first_moment}
\end{equation}

Combining
\cref{eq:momentum_error_decomposition,eq:tracking_component_bound,eq:noise_component_first_moment}
with the triangle inequality gives
\begin{align}
\mathbb E
\left\|
\overline M_{t-1}-\nabla f(W_{t-1})
\right\|_F
&\le
L\left(
\rho+\frac{\eta}{1-\mu}
\right)
+
\sqrt{\Gamma_t(\mu)}
\frac{\sigma}{\sqrt B}.
\label{eq:bias_corrected_momentum_error}
\end{align}
Applying
\cref{lem:achieved_linearization_deviation}
and using
$\mathcal E_t(\overline M_{t-1})
=
\mathcal E_t(M_{t-1})$, we obtain
\begin{align}
\mathbb E[
\mathcal E_t(M_{t-1})
]
&=
\mathbb E[
\mathcal E_t(\overline M_{t-1})
]
\nonumber\\
&\le
2\rho\,
\mathbb E
\left\|
\overline M_{t-1}-\nabla f(W_{t-1})
\right\|_F
\nonumber\\
&\le
2\rho
\left[
L\left(
\rho+\frac{\eta}{1-\mu}
\right)
+
\sqrt{
\frac{1-\mu}{1+\mu}
\frac{1+\mu^{t-1}}{1-\mu^{t-1}}
}
\frac{\sigma}{\sqrt B}
\right].
\label{eq:momusam_stability_proof}
\end{align}
This proves the finite-time bound
\cref{eq:momusam_stability}.

\subsubsection{Average stability under the prescribed schedule}

We first bound the average of
\(\sqrt{\Gamma_t(\mu)}\). Let $a:=1-\mu$.
For \(k=t-1\ge1\),
\begin{equation}
\Gamma_{k+1}(\mu)
=
\frac{a}{2-a}
\frac{1+(1-a)^k}{1-(1-a)^k}.
\label{eq:Gamma_a_form}
\end{equation}
For \(a\in(0,1)\), Bernoulli's inequality gives
$
(1-a)^{-k}
\ge
(1+a)^k
\ge
1+ak,
$
and therefore
$
(1-a)^k
\le
\frac{1}{1+ak}.
$
Consequently,
$
1-(1-a)^k
\ge
\frac{ak}{1+ak}.
$
Using \(1+(1-a)^k\le2\) and \(2-a\ge1\), we obtain
\begin{align}
\Gamma_{k+1}(\mu)
\le
a\frac{2(1+ak)}{ak}=
2a+\frac{2}{k}.
\label{eq:Gamma_simple_upper}
\end{align}
Hence,
$
\sqrt{\Gamma_{k+1}(\mu)}
\le
\sqrt{2a}
+
\sqrt{\frac{2}{k}}.
$
Since
\(
\sum_{k=1}^{N}k^{-1/2}\le2\sqrt N,
\)
we have
\begin{align}
\frac{1}{T-1}
\sum_{t=2}^{T}
\sqrt{\Gamma_t(\mu)}
&\le
\sqrt{2(1-\mu)}
+
\frac{2\sqrt2}{\sqrt{T-1}}.
\label{eq:average_Gamma_bound}
\end{align}

It remains to include the initial perturbation. Since
\(M_0=0\), \(\mathsf n(M_0)=0\), and \(\delta_1=0\),
\begin{align}
\mathcal E_1(M_0)
=
L_1(\nabla f(W_0))-L_1(M_0)
=
\rho\|\nabla f(W_0)\|_F.
\label{eq:initial_stability_gap}
\end{align}
For completeness, the gradient-gap inequality used below follows
directly from
\cref{ass:lower_bound,lipschiz}. The corresponding descent lemma gives
\begin{align}
f\left(
W_0-\frac{1}{L}\nabla f(W_0)
\right)
&\le
f(W_0)
-
\frac{1}{L}\|\nabla f(W_0)\|_F^2
+
\frac{1}{2L}\|\nabla f(W_0)\|_{\mathrm{op}}^2
\nonumber\\
&\le
f(W_0)
-
\frac{1}{2L}\|\nabla f(W_0)\|_F^2.
\end{align}
Since \(f\ge f_\star\), this implies
$
\|\nabla f(W_0)\|_F
\le
\sqrt{
2L\bigl(f(W_0)-f_\star\bigr)
}.
$
Therefore,
$
\frac{1}{T}
\mathbb E[\mathcal E_1(M_0)]
\le
\frac{\rho}{T}
\sqrt{
2L\bigl(f(W_0)-f_\star\bigr)
}.$

Averaging
\cref{eq:momusam_stability_proof} over \(t=2,\ldots,T\),
using \cref{eq:average_Gamma_bound}, and adding the initial term gives
\begin{align}
&\frac{1}{T}
\sum_{t=1}^{T}
\mathbb E[\mathcal E_t(M_{t-1})]\notag\\
&\le
\frac{\rho}{T}
\sqrt{
2L\bigl(f(W_0)-f_\star\bigr)
}+
\frac{2\rho(T-1)}{T}
\left[
L\left(
\rho+\frac{\eta}{1-\mu}
\right)
+
\frac{\sigma}{\sqrt B}
\left(
\sqrt{2(1-\mu)}
+
\frac{2\sqrt2}{\sqrt{T-1}}
\right)
\right].
\label{eq:momusam_average_complete}
\end{align}

Now choose $\eta=\rho=T^{-3/4}$ and $1-\mu=T^{-1/2}$.
It follows that
\begin{align}
\frac{1}{T}
\sum_{t=1}^{T}
\mathbb E[\mathcal E_t(M_{t-1})]
&=
\mathcal O\!\left(
\frac{L+\sigma/\sqrt B}{T}
+
\frac{\sigma}{\sqrt B\,T^{5/4}}
+
\frac{L}{T^{3/2}}
+
\frac{
\sqrt{L(f(W_0)-f_\star)}
}{T^{7/4}}
\right)
\nonumber\\
&=
\mathcal O\!\left(
\frac{L+\sigma/\sqrt B}{T}
\right).
\label{eq:momusam_average_final}
\end{align}

Finally, averaging
\cref{eq:musam_stability_proof} over \(t=1,\ldots,T\) gives
\begin{align}
\frac{1}{T}
\sum_{t=1}^{T}
\mathbb E\!\left[
\mathcal E_t(G_t^{\mathrm{clean}})
\right]
\le
\frac{2\rho\sigma}{\sqrt B}=
\mathcal O\!\left(
\frac{\sigma}{\sqrt B\,T^{3/4}}
\right),
\label{eq:musam_average_final}
\end{align}
where the last equality uses \(\rho=T^{-3/4}\).
Together,
\cref{eq:momusam_average_final,eq:musam_average_final}
prove
\cref{eq:momusam_stability_rate}.
\hfill\(\square\)

\newpage
\section{Additional Experiments}\label{experimentsec}

\begin{figure*}[t]
\centering
\begin{minipage}[b]{0.34\textwidth}
    \includegraphics[width=\textwidth]{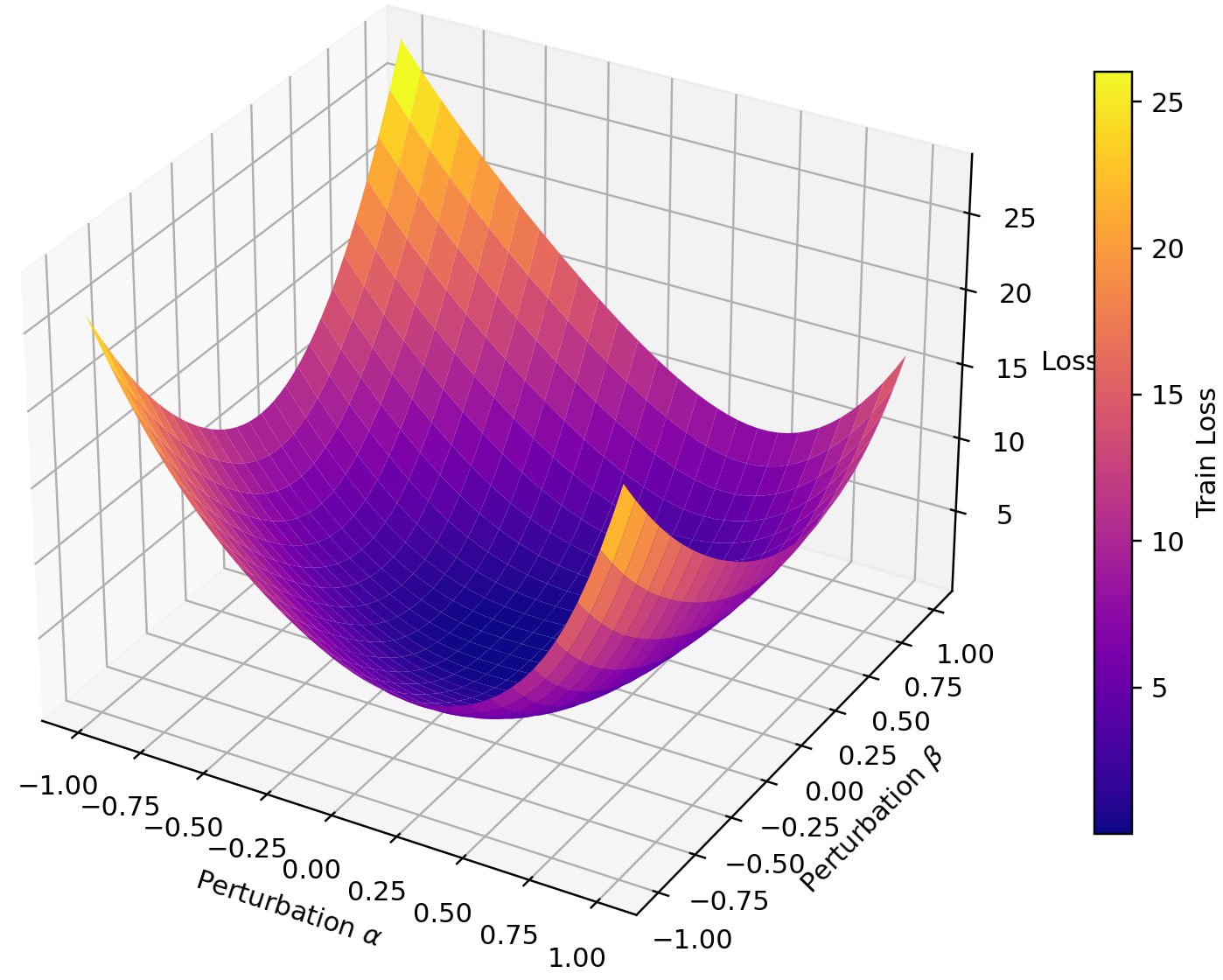}
    \par\vspace{-0.1cm} 
    \makebox[\textwidth]{{(a) \textsc{Muon}}}
\end{minipage}
\begin{minipage}[b]{0.32\textwidth}
    \includegraphics[width=\textwidth]{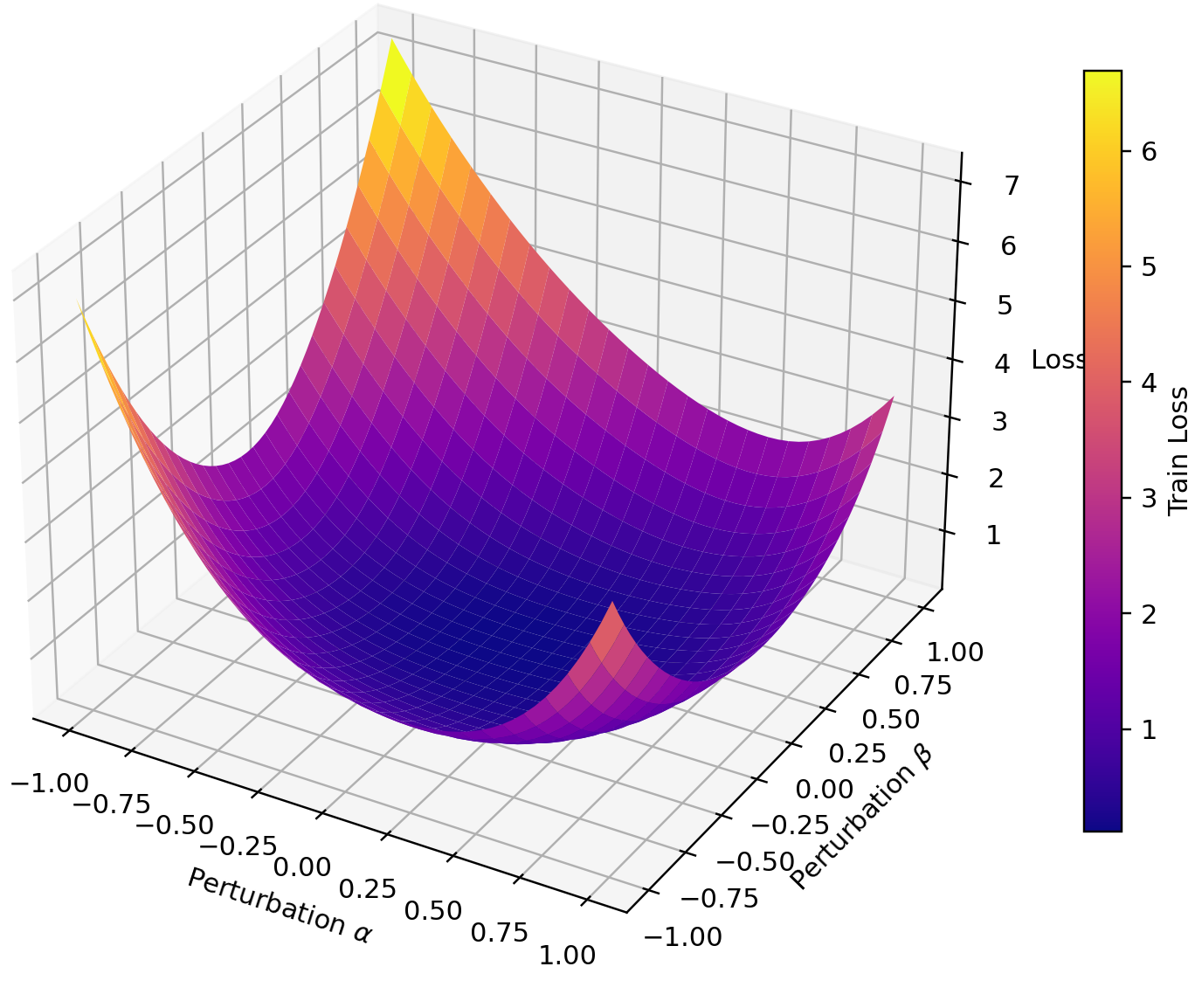}
    \par\vspace{-0.1cm} 
    \makebox[\textwidth]{{(b) \textsc{MuSAM}}}
\end{minipage}
\begin{minipage}[b]{0.32\textwidth}
    \includegraphics[width=\textwidth]{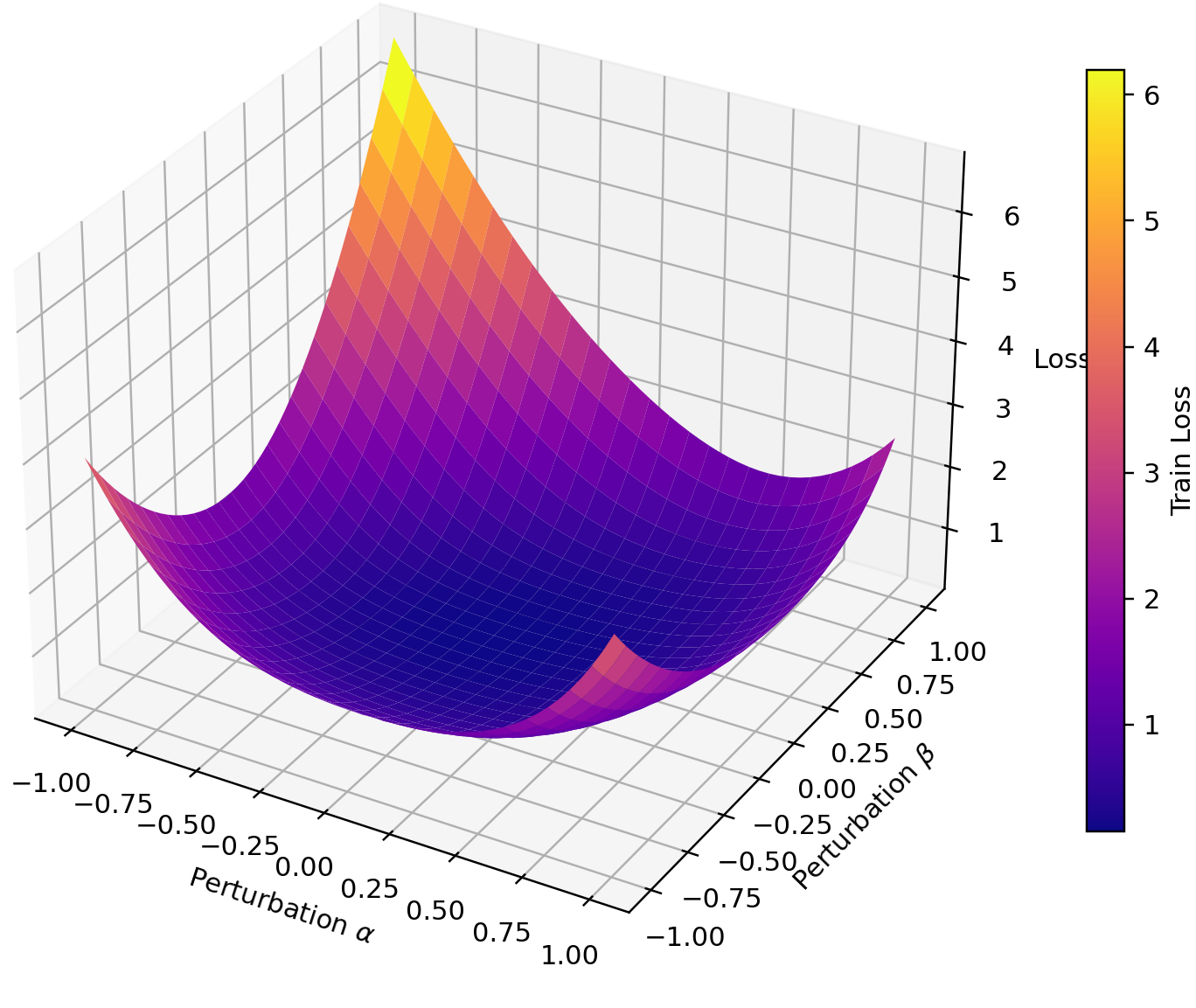}
    \par\vspace{-0.1cm} 
    \makebox[\textwidth]{{(c) \textsc{MoMuSAM}}}
\end{minipage}
\caption{3D loss landscapes for different algorithms.}
\label{fig:loss_landscape}
\end{figure*}

\begin{table}[t]
\centering
\caption{Test Accuracy under different datasets.}
\label{tab:test_acc}
\begin{tabular}{ccccccc}
\specialrule{1.9pt}{2pt}{2pt}
Dataset & \textsc{SAM} & \textsc{MSAM} & \textsc{VASSO} & \textsc{Muon} & \textsc{MuSAM}& \textsc{MoMuSAM} \\
\midrule
\textsc{FMNIST} & 93.83\textsubscript{$\pm$0.11} & {93.74\textsubscript{$\pm$0.14}} & {93.86\textsubscript{$\pm$0.08}} & {94.10\textsubscript{$\pm$0.06}} & {94.27\textsubscript{$\pm$0.08}} & {94.35\textsubscript{$\pm$0.17}} \\
\textsc{CIFAR-10} & 89.27\textsubscript{$\pm$0.05} & 89.40\textsubscript{$\pm$0.07} & 89.26\textsubscript{$\pm$0.25} & 90.72\textsubscript{$\pm$0.23} & 91.11\textsubscript{$\pm$0.18} & 91.21\textsubscript{$\pm$0.19} \\
\textsc{CIFAR-100} & 68.98\textsubscript{$\pm$0.29} & 68.48\textsubscript{$\pm$0.14} & 68.96\textsubscript{$\pm$0.33} & 74.18\textsubscript{$\pm$0.20} & 74.62\textsubscript{$\pm$0.17} & 75.14\textsubscript{$\pm$0.16} \\
\specialrule{1.9pt}{0pt}{0pt}
\end{tabular}
\end{table}

\paragraph{Experimental setup.} For the baseline optimizers, we tune the hyperparameters by grid search and report the best-performing configuration for each method. In contrast, for the proposed \textsc{MuSAM} and \textsc{MoMuSAM}, the hyperparameter choices are derived directly from the theoretical prescriptions in \cref{thm:musam_main,thm:momusam_main}. In particular, the resulting closed-form settings depend only on the total number of iterations \(T\), and are used throughout our experiments without additional task-specific search. This highlights an important practical advantage of our proposed methods: their hyperparameters are not only theoretically grounded, but also straightforward to instantiate in practice.
All image-classification experiments related to convergence and sensitivity analysis are repeated over 10 random seeds, and we report the mean and standard deviation. For loss-landscape visualization, we follow \citet{li2018visualizing} and plot the landscape along two random filter-normalized perturbation directions. The experiments were conducted on a host machine equipped with an Intel(R) Xeon(R) W9-3475X CPU running at 2.20\,GHz (maximum turbo frequency 4.80\,GHz), with 36 physical cores and 72 threads, 256\,GB DDR5 ECC RAM, and a single NVIDIA(R) RTX(TM) A6000 GPU with 48\,GB of memory.

\begin{table}[t]
\centering
\caption{Top 5 Hessian eigenvalues.}
\label{tab:hessian}
\begin{tabular}{cccccc}
\specialrule{1.9pt}{2pt}{2pt}
Algorithm & $\lambda_1$& $\lambda_2$ & $\lambda_3$ & $\lambda_4$ & $\lambda_5$ \\
\midrule
\textsc{Muon}  & 30.21 & 17.89 & 9.50 & 7.59 & 6.08 \\
\textsc{MuSAM} & 8.71 & 6.42 & 5.42 & 3.48 & 2.71 \\
\textsc{MoMuSAM} & 5.55 & 3.18 & 1.78 & 1.41 & 1.06 \\
\specialrule{1.9pt}{0pt}{0pt}
\end{tabular}
\end{table}

\begin{table}[t]
\centering
\setlength{\tabcolsep}{6pt}
\renewcommand{\arraystretch}{1.12}
\caption{Wall-clock time test accuracy of CIFAR-10 for different algorithms.}
\label{tab:wallclock}
\begin{tabular}{lccccccc}
\specialrule{1.9pt}{0pt}{0pt}
 & \multicolumn{6}{c}{\textbf{Wall-clock time test accuracy (\%)}} \\
\cmidrule(lr){2-8}
\textbf{Algorithm} & \textbf{200 sec} & \textbf{500 sec} & \textbf{800 sec} & \textbf{1100 sec} & \textbf{1400 sec} & \textbf{1700 sec}& \textbf{2000 sec} \\
\midrule
SAM    & 68.68 & 74.26 & 78.27 & 81.94 & 84.65 & 85.84 & 86.53 \\
MSAM   & 75.02 & 77.78 & 82.50 & 84.23 & 85.15 & 87.20 & 87.86 \\
VASSO  & 68.40  & 76.17 & 79.17 & 85.26 & 85.35 & 86.33 & 87.64 \\
Muon & 83.74  & 85.67 & 88.38 & 88.96 & 89.34 & 89.82& 90.02\\
MuSAM & 83.44  & 86.61 &87.57 & 89.10 & 89.49 & 89.60& 89.91\\
MoMuSAM & 84.30  & 87.27 & 88.45 & 89.22 & 89.67 & 89.93& 90.23
 \\
\specialrule{1.9pt}{2pt}{0pt}
\end{tabular}
\end{table}

\paragraph{Geometry analysis.}
The results in \cref{fig:loss_landscape}, \cref{tab:test_acc,tab:hessian} consistently support the central claim of our paper: introducing explicit sharpness awareness substantially strengthens Muon, and the momentum-guided design yields the strongest overall behavior. From \cref{tab:test_acc}, \textsc{Muon} already outperforms SAM-style baselines on all three image-classification benchmarks, confirming the benefit of matrix-aware orthogonalized updates; however, both \textsc{MuSAM} and \textsc{MoMuSAM} further improve upon \textsc{Muon} consistently, showing that matrix-aware update geometry and explicit sharpness control are indeed complementary. This gain becomes more pronounced on the harder dataset CIFAR-100, where \textsc{MoMuSAM} achieves the largest margin, suggesting that explicit local geometry control becomes increasingly valuable in more challenging and noisier optimization regimes. The geometry diagnostics in \cref{tab:hessian,fig:loss_landscape} explain this trend. Relative to \textsc{Muon}, \textsc{MuSAM} sharply reduces the top Hessian eigenvalues, and \textsc{MoMuSAM} reduces them even further, indicating substantially flatter local curvature around the converged solution. This is also visually reflected in the 3D loss landscapes: while \textsc{Muon} already reaches a relatively smooth basin, \textsc{MuSAM} produces a visibly flatter neighborhood, and \textsc{MoMuSAM} yields the broadest and smoothest basin among the three. Together, these results suggest that the accuracy gains are not merely optimization artifacts, but are tightly coupled with improved local flatness. Moreover, the consistent advantage of \textsc{MoMuSAM} over \textsc{MuSAM} indicates that the quality of the perturbation signal is critical: by constructing the perturbation from the momentum state rather than a single minibatch gradient, \textsc{MoMuSAM} probes sharpness with a direction that is temporally aggregated, less noisy, and more informed by Muon's accumulated optimization history, which in turn leads to a flatter solution and stronger final generalization.

\paragraph{Wall-clock time comparison.}
\cref{tab:wallclock} provides a more practical view of the efficiency--performance tradeoff by comparing test accuracy under the same wall-clock budget. The key observation is that \textsc{MoMuSAM} consistently matches or exceeds \textsc{Muon} across almost the entire time range, and achieves the best final accuracy overall. This is particularly significant because, unlike \textsc{MuSAM}, \textsc{MoMuSAM} does not require an additional gradient computation to construct the perturbation; instead, it reuses the momentum state already maintained by Muon, and therefore preserves essentially the same per-iteration computational profile as \textsc{Muon}. In other words, the gain of \textsc{MoMuSAM} is obtained without paying an additional wall-clock overhead relative to the base matrix-aware optimizer. By contrast, although \textsc{MuSAM} also improves generalization, its extra clean-gradient evaluation makes its advantage less pronounced under a fixed time budget. These results therefore highlight a central strength of our design: \textsc{MoMuSAM} is not only more accurate than \textsc{Muon}, but also achieves this improvement at nearly the same computational cost, yielding a distinctly better accuracy--efficiency tradeoff in practical training.

\end{document}